\documentclass{article}

\usepackage{graphicx} % more modern
\usepackage{subfigure} 
\usepackage{natbib}

\usepackage{algorithm}
\usepackage{algorithmic}

\usepackage{amsmath}
\usepackage{amssymb}

\usepackage{hyperref}
\usepackage{subfigure}

\usepackage{amsthm}

\newtheorem{theorem}{Theorem}
\newtheorem{lemma}[theorem]{Lemma} 
\newtheorem{proposition}[theorem]{Proposition}

\newtheorem{assumption}{Assumption}
\usepackage{defs}

\usepackage{authblk}

\begin{document} 

\title{Austerity in MCMC Land: Cutting the Metropolis-Hastings Budget}
\date{}

\author[1]{Anoop Korattikara \thanks{akoratti@ics.uci.edu}}
\author[1,2]{Yutian Chen  \thanks{yutian.chen@eng.cam.edu}}
\author[1,3]{Max Welling \thanks{welling@ics.uci.edu}}
\affil[1]{Department of Computer Science, University of California, Irvine}
\affil[2]{Department of Engineering, University of Cambridge}
\affil[3]{Informatics Institute, University of Amsterdam}

\renewcommand\Authands{ and }

\maketitle

\begin{abstract} 
Can we make Bayesian posterior MCMC sampling more efficient when faced with very large datasets?
We argue that computing the likelihood for $N$ datapoints in the Metropolis-Hastings (MH) test to reach a single binary decision is computationally inefficient. We introduce an approximate MH rule based on a sequential hypothesis test that allows us to accept or reject samples with high confidence using only a fraction of the data required for the exact MH rule. While this method introduces an asymptotic bias, we show that this bias can be controlled and is more than offset by a decrease in variance due to our ability to draw more samples per unit of time.
\end{abstract}

\section{Introduction}

Markov chain Monte Carlo (MCMC) sampling has been the main workhorse of Bayesian computation since the 1990s. A canonical MCMC algorithm proposes samples from a distribution $q$ and then accepts or rejects these proposals with a certain probability given by the Metropolis-Hastings (MH) formula \citep{metropolis1953equation,hastings1970monte}. For each proposed sample, the MH rule needs to examine the likelihood of all data-items. When the number of data-cases is large this is an awful lot of computation for one bit of information, namely whether to accept or reject a proposal.

In today's Big Data world, we need to rethink our Bayesian inference algorithms. Standard MCMC methods do not meet the Big Data challenge for the reason described above. Researchers have made some progress in terms of making MCMC more efficient, mostly by focusing on parallelization. Very few question the algorithm itself: is the standard MCMC paradigm really optimally efficient in achieving its goals? We claim it is not.

Any method that includes computation as an essential ingredient should acknowledge that there is a finite amount of time, $T$, to finish a calculation. An efficient MCMC algorithm should therefore decrease the ``error" (properly defined) maximally in the given time $T$. For MCMC algorithms, there are two contributions to this error: bias and variance. Bias occurs because the chain needs to burn in during which it is sampling from the wrong distribution. Bias usually decreases fast, as evidenced by the fact that practitioners are willing to wait until the bias has (almost) completely vanished after which they discard these ``burn-in samples". The second cause of error is sampling variance, which occurs because of the random nature of the sampling process. The retained samples after burn-in will reduce the variance as $O(1/T)$.

However, given a finite amount of computational time, it is not at all clear whether the strategy of retaining few unbiased samples and accepting an error dominated by variance is optimal. Perhaps, by decreasing the bias more slowly we could sample faster and thus reduce variance faster? In this paper we illustrate this effect by cutting the computational budget of the MH accept/reject step. To achieve that, we conduct sequential hypothesis tests to decide whether to accept or reject  a given sample and find that the majority of these decisions can be made based on a small fraction of the data with high confidence. A related method was used in \citet{singh2012monte}, where the factors of a graphical model are sub-sampled to compute fixed-width confidence intervals for the log-likelihood in the MH test.

Our ``philosophy" runs deeper than the algorithm proposed here. We advocate MCMC algorithms with a ``bias-knob", allowing one to dial down the bias at a rate that optimally balances error due to bias and variance. We only know of one algorithm that would also adhere to this strategy: stochastic gradient Langevin dynamics  \citep{WellingTeh11} and its successor stochastic gradient Fisher scoring \citep{AhnKorattikaraWelling12}. In their case the bias-knob was the stepsize. These algorithms do not have an MH step which resulted in occasional samples with extremely low probability. We show that our approximate MH step largely resolves this, still avoiding $O(N)$ computations per iteration.

In the next section we introduce the MH algorithm and discuss its drawbacks. Then in Section~\ref{sec:approxMCMC}, we introduce the idea of approximate MCMC methods and the bias variance trade-off involved. We develop approximate MH tests for Bayesian posterior sampling in Section~\ref{sec:approxMH} and present a theoretical analysis in Section~\ref{sec:analysis}. Finally, we show our experimental results in Section~\ref{sec:experiments} and conclude in Section~\ref{sec:conclusion}.

\section{The Metropolis-Hastings algorithm}\label{sec:MHalgorithm}

MCMC methods generate samples from a distribution $\mathcal{S}_0(\theta)$ by simulating a Markov chain designed to have stationary distribution $\mathcal{S}_0(\theta)$. A Markov chain with a given stationary distribution can be constructed using the Metropolis-Hastings algorithm \citep{metropolis1953equation,hastings1970monte}, which uses the following rule for transitioning from the current state $\theta_t$ to the next state $\theta_{t+1}$:

\begin{enumerate}
\item
 Draw a candidate state $\theta'$ from a proposal distribution $q(\theta'|\theta_t)$
 \item
 Compute the acceptance probability:
 \begin{equation}
 P_a = \min \left[ 1, \frac{\mathcal{S}_0(\theta') q(\theta_t|\theta')}{\mathcal{S}_0(\theta_{t})q(\theta'|\theta_t)} \right] \label{eqn:p_a}
 \end{equation}
 \item
Draw $u \sim \text{Uniform} [0,1]$. If $u < P_a$ set $\theta_{t+1} \leftarrow \theta'$, otherwise set $\theta_{t+1} \leftarrow \theta_t$.
\end{enumerate}

Following this transition rule ensures that the stationary distribution of the Markov chain is $\mathcal{S}_0(\theta)$. The samples from the Markov chain are usually used to estimate the expectation of a function $f(\theta)$ with respect to $\mathcal{S}_0(\theta)$. To do this we collect $T$ samples and approximate the expectation $I = \langle f \rangle_{\mathcal{S}_0}$ as $\hat{I} = \frac{1}{T} \sum_{t=1}^T f(\theta_t)$. Since the stationary distribution of the Markov chain is $\mathcal{S}_0$, $\hat{I}$ is an unbiased estimator of $I$ (if we ignore burn-in).

The variance of $\hat{I}$ is $V = \mathbb{E} [ (\langle f \rangle_{\mathcal{S}_0} - \frac{1}{T} \sum_{t=1}^T f(\theta_t))^2]$, where the expectation is over multiple simulations of the Markov chain. It is well known that $V \approx \sigma^2_{f,\mathcal{S}_0} \tau / T$, where $\sigma^2_{f,\mathcal{S}_0}$ is the variance of $f$ with respect to $S_0$ and $\tau$ is the integrated auto-correlation time, which is a measure of the interval between independent samples  \citep{gamerman2006markov}. Usually, it is quite difficult to design a chain that mixes fast and therefore, the auto-correlation time will be quite high. Also, for many important problems, evaluating $\mathcal{S}_0(\theta)$ to compute the acceptance probability $P_a$ in every step is so expensive that we can collect only a very small number of samples ($T$) in a realistic amount of computational time. Thus the variance of $\hat{I}$ can be prohibitively high, even though it is unbiased.

\section{Approximate MCMC and the Bias-Variance Tradeoff}\label{sec:approxMCMC}

Ironically, the reason MCMC methods are so slow is that they are designed to be unbiased. If we were to allow a small bias in the stationary distribution, it is possible to design a Markov chain that can be simulated cheaply \citep{WellingTeh11,AhnKorattikaraWelling12}.  That is, to estimate $I = \langle f \rangle_{\mathcal{S}_0}$, we can use a Markov chain with  stationary distribution $\mathcal{S}_\epsilon$ where $\epsilon$ is a parameter that can be used to control the bias in the algorithm.  Then $I$ can be estimated as $\hat{I} = \frac{1}{T} \sum_{t=1}^T f(\theta_t)$, computed  using samples from $\mathcal{S}_\epsilon$ instead of $\mathcal{S}_0$.

As $\epsilon \to 0$, $\mathcal{S}_\epsilon$ approaches $\mathcal{S}_0$ (the distribution of interest)  but it becomes expensive to simulate the Markov chain. Therefore, the bias in $\hat{I}$ is low, but the variance is high because we can collect only a small number of samples in a given amount of computational time. As $\epsilon$ moves away from $0$, it becomes cheap to simulate the Markov chain but the difference between $\mathcal{S}_\epsilon$ and $\mathcal{S}_0$ grows. Therefore, $\hat{I}$ will have higher bias, but lower variance because we can collect a larger number of samples in the same amount of computational time. This is a classical bias-variance trade-off and can be studied using the risk of the estimator.

The risk can be defined as the mean squared error in $\hat{I}$, i.e.\ $R = \mathbb{E} [ (I- \hat{I})^2]$, where the expectation is taken over multiple simulations of the Markov chain. It is easy to show that the risk can be decomposed as $R = B^2 + V$, where $B$ is the bias and $V$ is the variance. If we ignore  burn-in, it can be shown that $B = \langle f \rangle_{\mathcal{S}_\epsilon} -\langle f \rangle_{\mathcal{S}_0}$ and $V =\mathbb{E} [ (\langle f \rangle_{\mathcal{S}_\epsilon} - \frac{1}{T} f(\theta_t))^2] \approx \sigma^2_{f,\mathcal{S}_\epsilon} \tau / T$.

The optimal setting of $\epsilon$ that minimizes the risk depends on the amount of computational time available. If we have an infinite amount of computational time, we should set $\epsilon$ to 0. Then there is no bias, and the variance can be brought down to $0$  by drawing an infinite number of samples. This is the traditional MCMC setting. However, given a finite amount of computational time, this setting may not be optimal. It might be better to tolerate a small amount of bias in the stationary distribution if it allows us to reduce the variance quickly, either by making it cheaper to collect a large number of samples or by mixing faster.

It is interesting to note that two recently proposed algorithms follow this paradigm: Stochastic Gradient Langevin Dynamics (SGLD) \citep{WellingTeh11} and Stochastic Gradient Fisher Scoring (SGFS) \citep{AhnKorattikaraWelling12}. These algorithms are biased because they omit the required Metropolis-Hastings tests. However, in both cases, a knob $\epsilon$ (the step-size  of the proposal distribution) is available to control the bias.  As $\epsilon \to 0$, the acceptance probability $P_a \to 1$ and the bias from not conducting  MH tests disappears. However, when $\epsilon \to 0$ the chain mixes very slowly and the variance increases because the auto-correlation time $\tau \to \infty$. As $\epsilon$ is increased from $0$, the auto-correlation, and therefore the variance, reduces. But, at the same time, the acceptance probability reduces and the bias from not conducting MH tests increases as well.

In the next section, we will develop another class of approximate MCMC algorithms for the case where the target $\mathcal{S}_0$ is a Bayesian posterior distribution given a very large dataset. We achieve this by developing an approximate Metropolis-Hastings test, equipped with a knob for controlling the bias. Moreover, our algorithm has the advantage that it can be used with any proposal distribution. For example, our method allows approximate MCMC methods to be applied to problems where it is impossible to compute gradients (which is necessary to apply SGLD/SGFS). Or, we can even combine our method with SGLD/SGFS, to obtain the best of both worlds.

\section{Approximate Metropolis-Hastings Test for Bayesian Posterior Sampling} \label{sec:approxMH}

An important method in the toolbox of Bayesian inference is posterior sampling. Given a dataset of $N$ independent observations $X_N = \left\{x_1,\dots,x_N\right\}$, which we model using a distribution $p(x;\theta)$ parameterized by $\theta$, defined on a space $\Theta$ with measure $\Omega$, and a prior distribution $\rho(\theta)$, the task is to sample from the posterior distribution $\mathcal{S}_0(\theta) \propto \rho(\theta) \prod_{i=1}^N p(x_i;\theta)$.

If the dataset has a billion datapoints, it becomes very painful to compute $\mathcal{S}_0(.)$ in the MH test, which has to be done for each posterior sample we generate. Spending $O(N)$ computation to get just $1$ bit of information, i.e.\ whether to accept or reject a sample, is likely not the best use of computational resources.

But, if we try to develop accept/reject tests that satisfy detailed balance exactly with respect to the posterior distribution using only sub-samples of data, we will quickly see the no free lunch theorem kicking in. For example, the pseudo marginal MCMC method \citep{andrieu2009pseudo}  and the method developed by \citet{lin2000noisy} provide a way to conduct exact accept/reject tests using unbiased  estimators of the likelihood. However, unbiased estimators of the likelihood that can be computed from mini-batches of data, such as the Poisson estimator \citep{fearnhead2008particle} or the Kennedy-Bhanot estimator \citep{lin2000noisy} have very high variance for large datasets. Because of this, once we get a very high estimate of the likelihood, almost all proposed moves are rejected and the algorithm gets stuck.

Thus, we should be willing to tolerate some error in the stationary distribution if we want faster accept/reject tests. If we can offset this small bias by drawing a large number of samples cheaply and reducing the variance faster,  we can establish a potentially large reduction in the risk.

We will now show how to develop such approximate tests by reformulating the MH test as a statistical decision problem. It is easy to see that the original MH test (Eqn.~\ref{eqn:p_a}) is equivalent to the following procedure: Draw $u \sim \text{Uniform}[0,1]$ and accept the proposal $\theta'$ if the average difference $\mu$ in the log-likelihoods of $\theta'$ and $\theta_{t}$ is greater than a threshold $\mu_0$, i.e.\ compute
\begin{align}
\mu_0 &= \frac{1}{N} \log \left[ u  \frac{ \rho(\theta_t) q(\theta'|\theta_t) } { \rho(\theta')  q(\theta_t|\theta')} \right], \text{~~and~~} \label{eqn:mu_0} \\
\mu &= \frac{1}{N} \sum_{i=1}^N l_i \text{~~where~~} l_i = \log p(x_i;\theta') -  \log p(x_i;\theta_t)
\end{align}
Then if $\mu > \mu_0$, accept the proposal and set $\theta_{t+1} \leftarrow \theta'$. If  $\mu \leq \mu_0$, reject the proposal and set $\theta_{t+1} \leftarrow \theta_{t}$. This reformulation of the MH test makes it very easy to frame it as a statistical hypothesis test. Given $\mu_0$ and a random sample $\left\{ l_{i_1}, \dots,  l_{i_n} \right\}$ drawn without replacement from the population $\left\{ l_1, \dots, l_N \right\}$, can we decide whether the population mean $\mu$ is greater than or less than the threshold $\mu_0$? The answer to this depends on the precision in the random sample. If the difference between the sample mean $\bar{l}$ and $\mu_0$ is significantly greater than the standard deviation $s$ of $\bar{l}$, we can make the decision to accept or reject the proposal confidently. If not, we should draw more data to increase the precision of $\bar{l}$ (reduce $s$)  until we have enough evidence to make a decision.

More formally, we test the hypotheses $H_1: \mu > \mu_0$ vs $H_2: \mu < \mu_0$. To do this, we proceed as follows: We compute the sample mean $\bar{l}$ and the sample standard deviation $s_l=\sqrt{(\overline{l^2}-(\bar{l})^2)\frac{n}{n-1}}$
%$s_l=\sqrt{\sum_{i=1}^n (l_i - \bar{l})^2 / (n-1)}$
. Then the standard deviation of $\bar{l}$ can be estimated as:
\begin{equation}
s= \frac{s_l}{\sqrt{n}} \sqrt{1- \frac{n-1}{N-1}} \label{eqn:estimated_std}
\end{equation}
where $\sqrt{1- \frac{n-1}{N-1}}$, the finite population correction term, is applied because we are drawing the subsample without replacement from a finite-sized population. Then, we compute the test statistic:
\begin{equation}
t = \frac{\bar{l} - \mu_0}{s} \label{eqn:t-statistics}
\end{equation}

\begin{algorithm}
\caption{Approximate MH test}\label{alg:approxmh}
\begin{algorithmic}[1]
  \REQUIRE $\theta_t$, $\theta'$, $\epsilon$, $\mu_0$, $X_N$, $m$
  \ENSURE $accept$
  \STATE Initialize estimated means $\bar{l} \leftarrow 0$ and $\overline{l^2} \leftarrow 0$
  \STATE Initialize $n \leftarrow 0$, $done \leftarrow$ \textbf{false}
  \STATE Draw $u \sim $ Uniform[0,1]
  \WHILE {\textbf{not} $done$}
    \STATE Draw mini-batch $\mathcal{X}$ of size min ($m$, $N-n$) without replacement from $X_N$ and set $X_N \leftarrow X_N \setminus \mathcal{X}$
    \STATE Update $\bar{l}$ and $\overline{l^2}$ using $\mathcal{X}$, and $n\leftarrow n+|\cX|$
		\STATE Estimate std $s$ using Eqn.~\ref{eqn:estimated_std}% $ s \leftarrow \sqrt{1 - \dfrac{n-1}{N-1}}  \sqrt{\dfrac{\bar{l^2} - (\bar{l})^2 }{n-1}}$
		\STATE Compute $\delta \leftarrow 1 - \phi_{n-1}\left(\left|\dfrac{\bar{l} - \mu_0}{s}\right|\right)$
		\IF{$\delta < \epsilon$}
 		  \STATE  $accept \leftarrow \textbf{true}$ if $\bar{l}>\mu_0$ and $\textbf{false}$ otherwise
   		\STATE $done \leftarrow \textbf{true}$
    \ENDIF
  \ENDWHILE
\end{algorithmic}
\end{algorithm}

If $n$ is large enough for the central limit theorem (CLT) to hold, the test statistic $t$ follows a standard Student-t distribution with $n-1$ degrees of freedom, when $\mu=\mu_0$ (see Fig.~\ref{fig:tstat_normality} in supplementary for an empirical verification). Then, we compute $\delta = 1-\phi_{n-1}(|t|)$ where $\phi_{n-1}(.)$ is the cdf of the standard Student-t distribution with $n-1$ degrees of freedom. If $\delta < \epsilon$ (a fixed threshold) we can  confidently say that $\mu$ is significantly different from $\mu_0$. In this case, if $\bar{l} > \mu_0$, we decide $\mu > \mu_0$, otherwise we decide $\mu <\mu_0$. If $\delta \geq \epsilon$, we do not have enough evidence to make a decision. In this case, we draw more data to reduce the uncertainty, $s$, in the sample mean $\bar{l}$.  We keep drawing more data until we have the required confidence (i.e.\ until $\delta < \epsilon$). Note, that this procedure will terminate because when we have used all the available data, i.e.\ $n=N$, the standard deviation $s$ is 0, the sample mean $\bar{l} = \mu$ and $\delta = 0 <\epsilon$. So, we will make the same decision as the original MH test would make. Pseudo-code for our test is shown in Algorithm~\ref{alg:approxmh}. Here, we start with a mini-batch of size $m$ for the first test and increase it by $m$ datapoints when required.  
%This algorithm is similar to the Pocock sequential design \citep{pocock1977group} which allows us to set a value of $\epsilon$ given a particular value of Type I error.

The advantage of our method is that often we can make confident decisions with $n<N$ datapoints and save on computation, although we introduce a small bias in the stationary distribution. But, we can use the computational time we save to draw more samples and reduce the variance. The bias-variance trade-off can be controlled by adjusting the knob $\epsilon$. When $\epsilon$ is high, we make decisions without sufficient evidence and introduce a high bias. As $\epsilon \to 0$, we make more accurate decisions but are forced to examine more data which results in high variance.

Our algorithm will behave erratically if the CLT does not hold, e.g.\ with very sparse datasets or datasets with extreme outliers. The CLT assumption can be easily tested empirically before running the algorithm to avoid such pathological situations. The sequential hypothesis testing method can also be used to speed-up Gibbs sampling in densely connected Markov Random Fields. We explore this idea briefly in Section~\ref{sec:gibbs} of the supplementary.

\section{Error Analysis and Test Design}\label{sec:analysis}

In~\ref{sec:erroranalysis}, we study the relation between the parameter $\epsilon$,  the error $\cE$ of the complete sequential test,  the error $\Delta$ in the acceptance probability and the error in the stationary distribution. In~\ref{sec:testdesign}, we describe how to design an optimal test that minimizes data usage given a bound on the error.

\subsection{Error Analysis and Estimation}\label{sec:erroranalysis}

The parameter $\epsilon$ is an upper-bound on the error of a single test and not the error of the complete sequential test. To compute this error, we assume a)  $n$ is large enough that the $t$ statistics can be approximated with $z$ statistics, and b) the joint distribution of the $\bar{l}$'s corresponding to different mini-batches used in the test is multivariate normal. Under these assumptions, we can show that the test statistic at different stages of the sequential test follows a Gaussian Random Walk process. This allows us to compute the error of the sequential test $\cE(\mu_{\mathrm{std}},m,\epsilon)$, and the expected proportion of the data required to reach a decision $\bar{\pi}(\mu_{\mathrm{std}},m,\epsilon)$, using an efficient dynamic programming algorithm. Note that $\cE$ and $\bar{\pi}$ depend on $\theta$, $\theta'$ and $u$ only through the `standardized mean' defined as $\mu_{\mathrm{std}}(u,\theta,\theta') \overset{\mathrm{def}}{=} \dfrac{\left( \mu(\theta,\theta') - \mu_0(\theta,\theta',u) \right)\sqrt{N-1}}{\sg_l(\theta,\theta')}$ where $\sigma_l$ is the true  standard deviation of the $l_i$'s.  See Section~\ref{sec:gaussian_process} of the supplementary for a detailed derivation and an empirical validation of the assumptions. 

\begin{figure}
\centering
  \includegraphics[width=.8\linewidth]{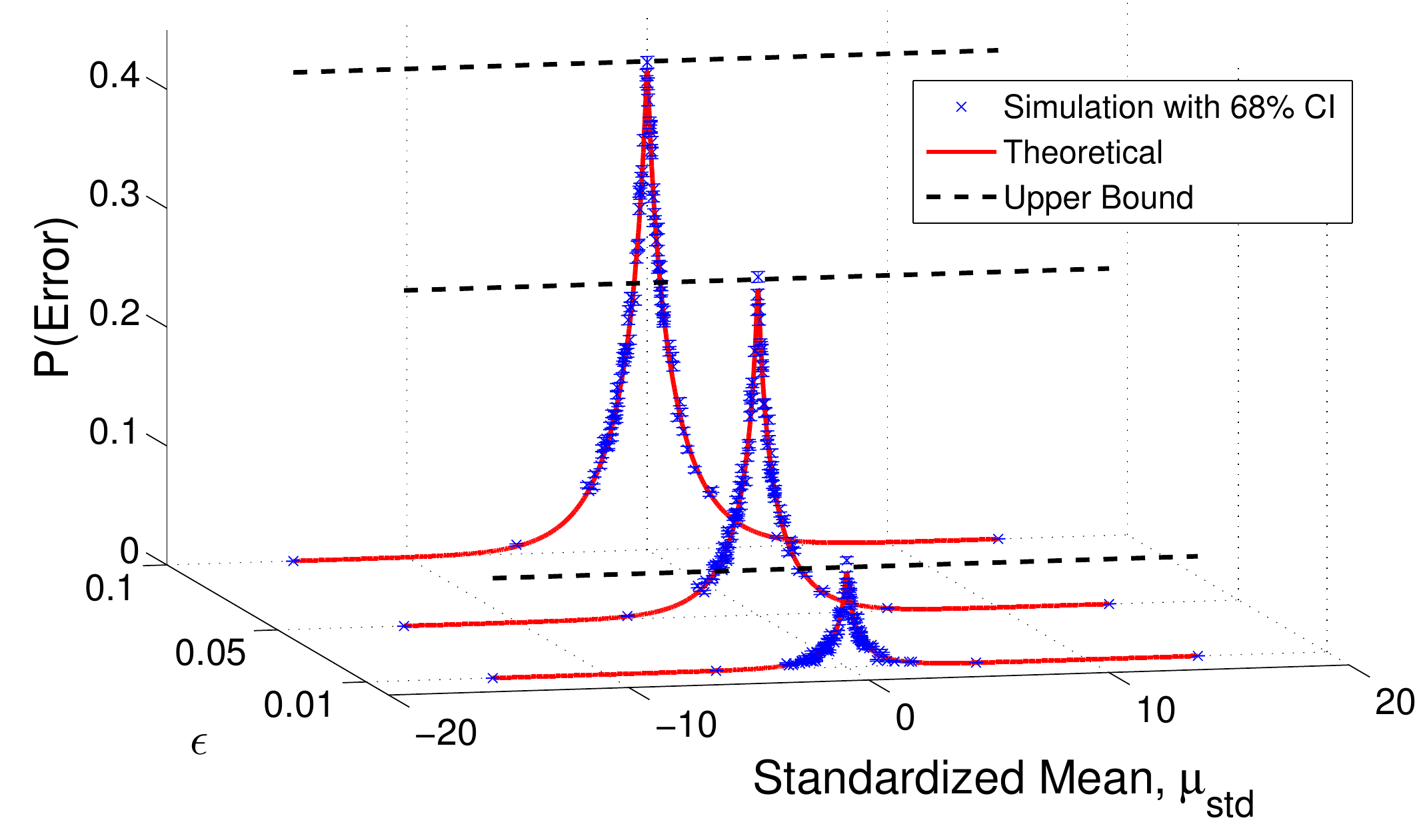}
  \caption{Error $\cE$ estimated using simulation (blue cross with 1 $\sigma$ error bar) and dynamic programming (red line). An upper bound (black dashed line) is also shown.}
  \label{fig:dynprog_vs_sim_error}
\end{figure}

Fig.~\ref{fig:dynprog_vs_sim_error} shows the theoretical and actual error of $1000$ sequential tests for the logistic regression model described in Section~\ref{sec:exp_mnist}. The error $\cE(\mu_{\mathrm{std}},m,\epsilon)$ is highest in the worst case when $\mu = \mu_0$. Therefore, $\cE(0,m,\epsilon)$ is an upper-bound on $\cE$. Since the error decreases sharply as $\mu$ moves away from $\mu_0$, we can get a more useful estimate of $\cE$ if we have some knowledge about the distribution of $\mu_{\mathrm{std}}$'s that will be encountered during the Markov chain simulation. 
 
% \begin{assumption}\label{ass:normal}
%The joint distribution of the sequence $(\bar{l}_1,\bar{l}_2,\dots)$ follows a multi-variate normal distribution.
%\end{assumption}
%\begin{assumption}\label{ass:exact_sg}
%$s_{l} = \sg_l$, where $\sg_l=\mathrm{std}(\{l_i\})$ 
%\end{assumption}
%The second assumption infers that the $t$ statistics reduces to $z$ statistics, and we obtain the distribution of the sequence of statistics $\bz=(z_1,z_2,\dots)$ across multiple tests in the following proposition:
%\begin{proposition}\label{prop:gaussian_process}
%Given Assumption~\ref{ass:normal} and~\ref{ass:exact_sg}, the sequence $\bz$ follows a \emph{Gaussian random walk process}:
%\begin{equation}
%P(z_j|z_1,\dots,z_{j-1})=\cN(m_j(z_{j-1}), \sg_{z,j}^2)
%\end{equation}
%where
%\begin{align}
%m_j(z_{j-1}) &= \mu_{\mathrm{std}} \dfrac{\pi_j - \pi_{j-1}}{1-\pi_{j-1}} \dfrac{1}{\sqrt{\pi_j(1-\pi_j)}} \nn\\
%& + z_{j-1}\sqrt{\dfrac{\pi_{j-1}}{\pi_j}\dfrac{1-\pi_j}{1-\pi_{j-1}}} \label{eqn:cond_mean} \\
%\sg_{z,j}^2 &= \dfrac{\pi_j - \pi_{j-1}}{\pi_j(1-\pi_{j-1})} \label{eqn:cond_var}
%\end{align}
%with $\mu_{\mathrm{std}} = \frac{(\mu - \mu_0)\sqrt{N-1}}{\sg_l}$ being the standardized mean, and $\pi_j=jm/N$ the proportion of data up to $j$-th mini-batch.
%\end{proposition}

Now, let $P_{a,\epsilon}(\theta, \theta')$ be the actual acceptance probability of our algorithm and let  $\Delta(\theta,\theta') \defeq P_{a,\epsilon}(\theta, \theta') - P_a(\theta, \theta')$ be the error in $P_{a,\epsilon}$. In Section~\ref{sec:MH_step_error} of the supplementary, we show that for any $(\theta,\theta')$:
\begin{equation}
\Delta = \int_{P_a}^1 \cE (\mu_{\mathrm{std}}(u)) \td u - \int_{0}^{P_a} \cE (\mu_{\mathrm{std}}(u)) \td u
\end{equation}
Thus, the errors corresponding to different $u$'s partly cancel each other. As a result, although $| \Delta(\theta,\theta') |$ is upper-bounded by the worst-case error $\cE(0,m,\epsilon)$ of the sequential test, the actual error is usually much smaller. For any given $(\theta,\theta')$, $\Delta$ can be computed easily using 1-dimensional quadrature.

Finally, we show that the error in the stationary distribution is bounded linearly by $\Delta_{\text{max}}=\sup_{\theta,\theta'}|\Delta(\theta,\theta')|$. As noted above, $\Delta_{\text{max}} \leq \cE(0,m,\epsilon)$ but is usually much smaller. Let $d_v(P,Q)$ denote the total variation distance\footnote{The total  variation distance between two distributions $P$ and $Q$, that are absolutely continuous w.r.t.\ measure $\Omega$, is defined as $d_v(P,Q)\overset{\mathrm{def}}{=}\ha\int_{\theta\in\Theta}|f_P(\theta)-f_Q(\theta)|d\Omega(\theta)$ where $f_P$ and $f_Q$ are their respective densities (or  Radon-Nikodym derivatives to be more precise).} between two distributions, $P$ and $Q$. If the transition kernel $\mathcal{T}_0$ of the exact Markov chain satisfies the contraction condition $d_v(P\mathcal{T}_0, \mathcal{S}_0) \leq \eta d_v(P, \mathcal{S}_0)$ for all probability distributions $P$ with a constant $\eta\in[0,1)$, we can prove (see supplementary Section~\ref{sec:proof}) the following upper bound on the error in the stationary distribution:
\begin{theorem}\label{thm:bound}
The distance between the posterior distribution $\mathcal{S}_0$ and the stationary distribution of our approximate Markov chain $\mathcal{S}_{\epsilon}$ is upper bounded as:
$$
d_v(\mathcal{S}_0, \mathcal{S}_\epsilon) \leq \frac{\Delta_{\text{max}}}{1-\eta}
$$
\end{theorem}
\subsection{Optimal Sequential Test Design}\label{sec:testdesign}

We now briefly describe how to choose the parameters of the algorithm: $\epsilon$, the error of a single test and $m$, the mini-batch size. A very simple strategy we recommend is to choose $m \approx 500$ so that the Central Limit Theorem holds and keep $\epsilon$ as small as possible while maintaining a low average data usage. This rule works well in practice and is used in Experiments~\ref{sec:exp_mnist} - \ref{sec:exp_sgld}. 

The more discerning practitioner can design an optimal test that minimizes the data used while keeping the error below a given tolerance. Ideally, we want to do this based on a tolerance on the error in the stationary distribution $\cS_\epsilon$. Unfortunately, this error depends on the contraction parameter, $\eta$, of the exact transition kernel, which is difficult to compute. A more practical choice is a bound on the error $\Delta$ in the acceptance probability, since the error in $\cS_\epsilon$ increases linearly with $\Delta$. Since $\Delta$ is a function of $(\theta, \theta')$, we can try to control the average value of $\Delta$ over the empirical distribution of $(\theta,\theta')$ that would be encountered while simulating the Markov chain. Given a tolerance $\Delta^*$ on this average error, we can find the optimal $m$ and $\epsilon$ by solving the following optimization problem (e.g.\ using grid search) to minimize the average data usage :
\begin{align}
&\min_{m, \epsilon} \eE_{\theta, \theta'} \left[\eE_{u}\bar{\pi}(\mu_{\text{std}}(u,\theta,\theta'), m, \epsilon)\right] \nn\\
& \text{s.t. } \eE_{\theta, \theta'}  |\Delta(m, \epsilon,\theta,\theta')| \leq \Delta^* \label{eqn:avg_design}
\end{align}
In the above equation, we estimate the average data usage, $\eE_u[\bar{\pi}]$, and the error in the acceptance probability, $\Delta$, using dynamic programming with one dimensional numerical quadrature on $u$. The empirical distribution for computing the expectation with respect to $(\theta,\theta')$ can be obtained using a trial run of the Markov chain. 
%The expectation of data usage $\bar{\pi}$ with respect to $u$ can be computed efficiently using numerical quadrature. 
Without a trial run the best we can do is to control the worst case error  $\cE(0, m, \epsilon)$ (which is also an upper-bound on $\Delta$) in each sequential test by solving the following minimization problem:
\begin{align}
&\min_{m, \epsilon} \bar{\pi}(0, m, \epsilon)  ~~\text{s.t.}~~ \cE(0, m, \epsilon) \leq \Delta^* \label{eqn:worstcase_design}
\end{align}
But this leads to a very conservative design as the worst case error is usually much higher than the average case error. We illustrate the sequential design in Experiment~\ref{sec:optdesign_exp}. More details and a generalization of this method is given in supplementary Section~\ref{sec:optimal_design}.

\section{Experiments}\label{sec:experiments}

\subsection{Random Walk - Logistic Regression} \label{sec:exp_mnist}

%\begin{figure}[ht]
%  \centering
%  \subfigure[$\epsilon$=0, WCT = 50 sec, T = 9.5K]{\includegraphics[scale=0.20]{MNIST_Marginal_T50_eps0}}\quad
%  \subfigure[$\epsilon$=0.01, WCT = 50 sec, T = 16.5K]{\includegraphics[scale=0.20]{MNIST_Marginal_T50_eps0p01}}\\
% \subfigure[$\epsilon$=0, WCT = 100 sec, T = 19K]{\includegraphics[scale=0.20]{MNIST_Marginal_T100_eps0}}\quad
%\subfigure[$\epsilon$=0.01, WCT = 100 sec, T = 33K]{\includegraphics[scale=0.20]{MNIST_Marginal_T100_eps0p01}}\\
%\subfigure[$\epsilon$=0, WCT = 400 sec, T = 75.5K]{\includegraphics[scale=0.20]{MNIST_Marginal_T400_eps0}}\quad
% \subfigure[$\epsilon$=0.01, WCT = 400 sec, T = 133K]{\includegraphics[scale=0.20]{MNIST_Marginal_T400_eps0p01}}\quad
%  \caption{Marginals $\theta_1$ vs $\theta_t$ for $\epsilon = 0$ and $\epsilon = 0.01$ at different values of wall clock time. The actual number of samples $T$ is shown as well.}
%  \label{fig:MNIST_marginals}
%\end{figure}
We first test our method using a random walk proposal $q(\theta'|\theta_t) = \mathcal{N}(\theta_t,\sigma_{RW}^2)$. Although the random walk proposal is not efficient, it is very useful for illustrating our algorithm because the proposal does not contain any information about the target distribution, unlike Langevin or Hamiltonian methods. So, the responsibility of converging to the correct distribution lies solely with the MH test. Also since $q$ is symmetric, it does not appear in the MH test and we can use $\mu_0 = \frac{1}{N} \log \left[ u   \rho(\theta_t)  / \rho(\theta') \right]$.

The target distribution in this experiment was the posterior for a logistic regression model trained on the MNIST dataset for classifying digits 7 vs 9. The dataset consisted of 12214 datapoints and we reduced the dimensionality from 784 to 50 using PCA. We chose a zero mean spherical Gaussian prior with precision =  10, and set $\sigma_{RW} = 0.01$.

\begin{figure}
\centering
\includegraphics[scale=0.4]{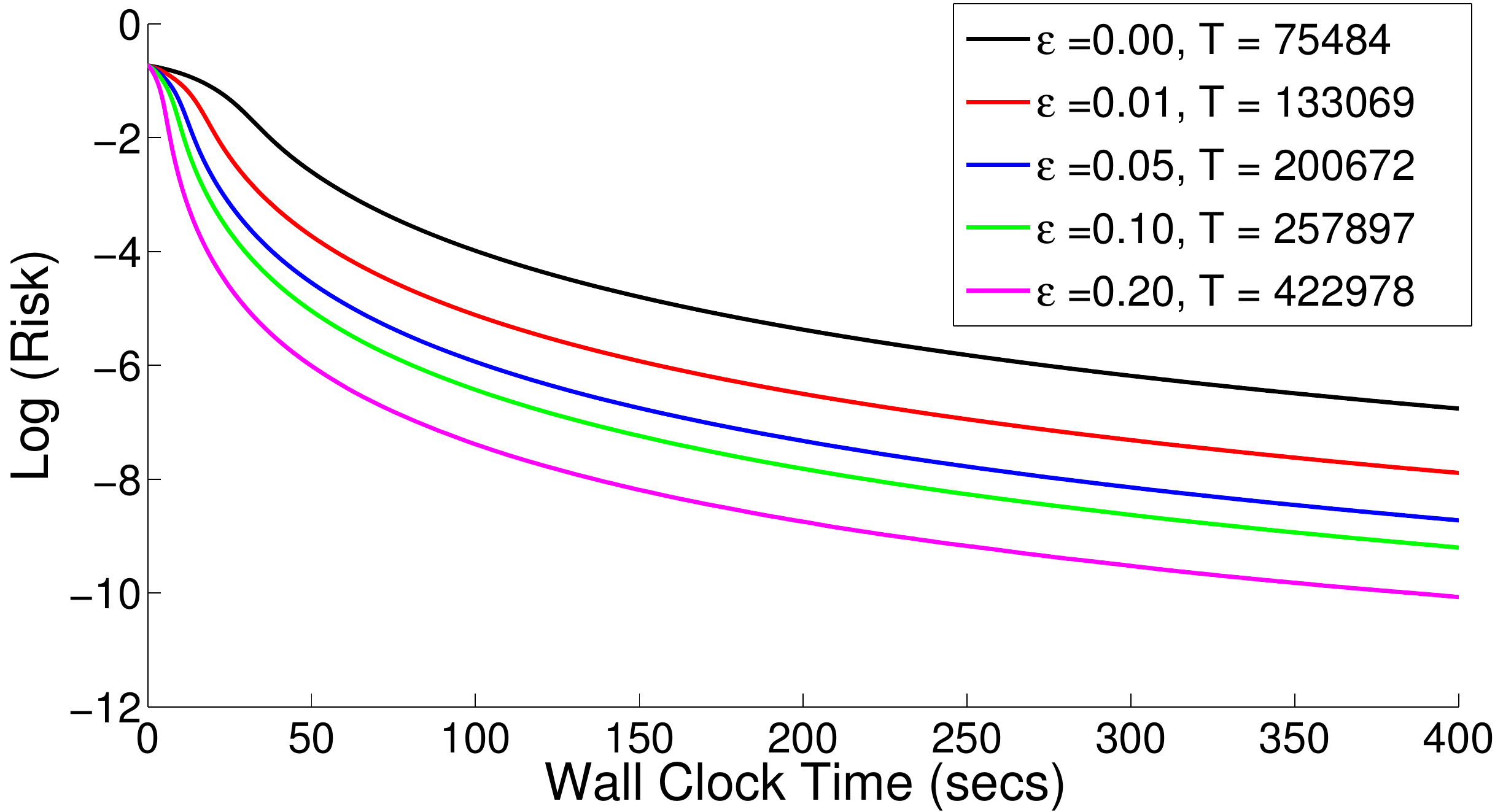}  
\caption{Logistic Regression: Risk in predictive mean.}
\label{fig:MNIST_LogRiskvsT}
\end{figure}

%
%In figure~\ref{fig:MNIST_marginals}, we show marginals of $\theta_1$ vs $\theta_2$ for $\epsilon=0$ (exact MH) and $\epsilon=0.01$ at different values of wall clock time. The red curves are marginals for 3 different runs of the random walk algorithm whereas the blue curve shows the true marginal obtained from a  long run of Hybrid Monte Carlo.   After 50 secs, the exact MH algorithm $\epsilon=0$ has still not completed burn-in. Our algorithm with $\epsilon = 0.01$ is able to accelerate burn-in because it can collect more samples in a given amount of time. Theoretically, as more computational time becomes available the exact MH algorithm will catch up with (and eventually outperform) our algorithm,  because the exact MH algorithm is unbiased unlike ours. But the bias is hardly noticeable in this example and the error is still dominated by the variance even after collecting around 100K samples.

In Fig.~\ref{fig:MNIST_LogRiskvsT}, we show how the logarithm of the risk in estimating the predictive mean, decreases as a function of wall clock time.  The predictive mean of a test point $x^*$ is defined as $\mathbb{E}_{p(\theta|X_N)} [ p(x^*|\theta) ]$. To calculate the risk, we first estimate the true predictive mean using a long run of Hybrid Monte Carlo. Then, we compute multiple estimates of the predictive mean from our approximate algorithm and obtain the risk as the mean squared error in these estimates. We plot the average risk of 2037 datapoints in the test set.  Since the risk $R =B^2 + V =  B^2 + \frac{\sigma^2 f}{T}$, we  expect it to decrease as a function of time until the bias dominates the variance. The figure shows that even after collecting a lot of samples, the risk is still dominated by the variance and the minimum risk is obtained with $\epsilon>0$. 

\subsection{Independent Component Analysis}\label{sec:exp_ica}

Next, we use our algorithm to sample from the posterior distribution of the unmixing matrix in Independent Component Analysis (ICA) \citep{hyvarinen2000independent}. When using prewhitened data, the unmixing matrix $W \in \mathbb{R}^{D \times D}$ is constrained to lie on the Stiefel manifold of orthonormal matrices.  We choose a  prior that is uniform over the manifold and zero elsewhere. We model the data as $p(x|W) = | \text{det}(W)| \prod_{j=1}^D \left[4 \cosh^2 (\frac{1}{2} w_j^Tx)\right]^{-1}$ where $w_j$ are the rows of $W$. Since the prior is zero outside the manifold, the same is true for the posterior. Therefore we use a random walk on the Stiefel manifold as a proposal distribution \citep{ouyang2008bayesian}. Since this is a symmetric proposal distribution, it does not appear in the MH test and we can use $\mu_0 = \frac{1}{N} \log \left[ u  \right]$.

To perform a large scale experiment, we created a synthetic dataset by mixing 1.95 million samples of 4 sources: (a) a Classical music recording (b) street / traffic noise (c) \& (d) 2 independent Gaussian sources. To measure the correctness of the sampler, we measure the risk in estimating $I = \mathbb{E}_{p(W|X)} \left[ d_A(W,W_0) \right]$ where the test function $d_A$ is the Amari distance \citep{amari1996new} and $W_0$ is the true unmixing matrix. We computed the ground truth using a long run (T = 100K samples) of the exact MH algorithm. Then we ran each algorithm 10 times, each time for $\approx$ 6400 secs. We calculated the risk by averaging the squared error in the estimate from each Markov chain, over the 10 chains. This is shown in Fig.~\ref{fig:ICA_risk}.  Note that even after 6400 secs the variance dominates the bias, as evidenced by the still decreasing risk, except for the most biased algorithm with $\epsilon = 0.2$. Also, the lowest risk at 6400 secs is obtained with $\epsilon = 0.1$ and not the exact MH algorithm ($\epsilon = 0$). But we expect the exact algorithm to outperform all the approximate algorithms if we were to run for an infinite time.

\begin{figure}
\centering
\includegraphics[scale=0.4]{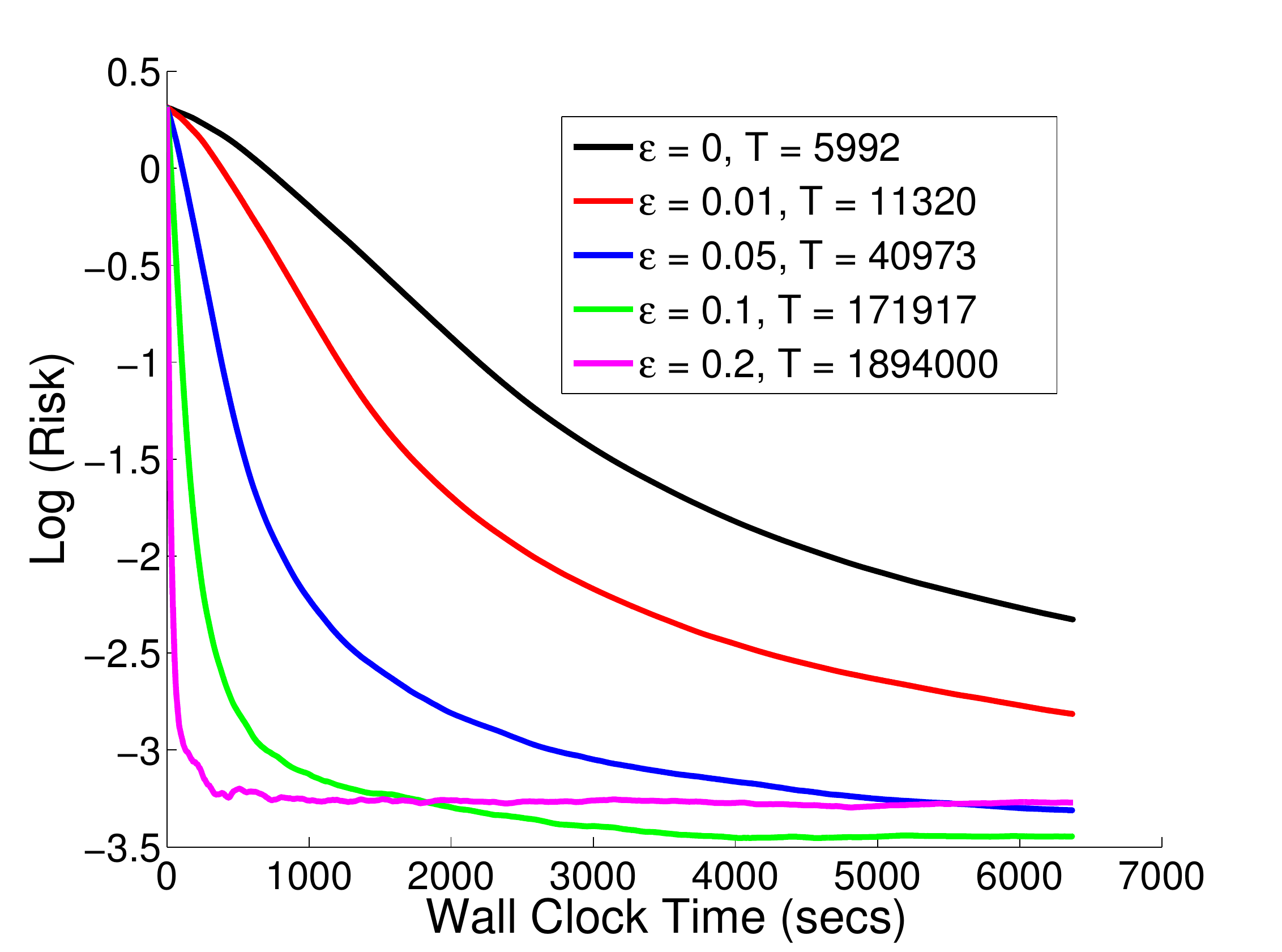}
\caption{ICA: Risk in mean of Amari distance}
\label{fig:ICA_risk}
\end{figure}

\subsection{Variable selection in Logistic Regression}\label{sec:exp_rjmcmc}

Now, we apply our MH test to variable selection in a logistic regression model using the reversible jump MCMC algorithm of \citet{green1995reversible}. We use a model that is similar to the Bayesian LASSO model for linear regression described in \citet{chen2011bayesian}. Specifically, given $D$ input features, our parameter $\theta = \left\{ \beta, \gamma \right\}$ where $\beta$ is a vector of $D$ regression coefficients and $\gamma$ is a $D$ dimensional binary vector that indicates whether a particular feature is included in the model or not. The prior we choose for $\beta$ is $p(\beta_j|\gamma, \nu) = \frac{1}{2\nu} \exp \left\{ - \frac{|\beta_j|}{\nu}\right\}$ if $\gamma_j = 1$. If $\gamma_j =0$,  $\beta_j$ does not appear in the model. Here $\nu$ is a shrinkage parameter that pushes $\beta_j$ towards 0, and we choose a prior $p(\nu) \propto 1/\nu$. We also place a right truncated Poisson prior $p(\gamma|\lambda) \propto \dfrac{\lambda^k}{ {D \choose k} k!}$ on $\gamma$ to control the size of the model, $k = \sum_{j=1}^D \gamma_j$ We set $\lambda = 10^{-10}$ in this experiment.

Denoting the likelihood of the data by $l_N(\beta,\gamma)$, the posterior distribution after integrating out $\nu$ is $p(\beta,\gamma|X_N,y_N,\lambda) \propto l_N(\beta,\gamma) \|\beta \|_{1}^{-k} \lambda^k B(k,D-k+1)$ where $B(.,.)$ is the beta function.  Instead of integrating out $\lambda$, we use it as a parameter to control the size of the model. We use the same proposal distribution as in \cite{chen2011bayesian} which is a mixture of 3 type of moves that are picked randomly in each iteration: an update move, a birth move and a death move. A detailed description is given in Supplementary Section~\ref{sec:rjmcmc_sup}.

We applied this to the MiniBooNE dataset from the UCI machine learning repository\citep{Bache+Lichman:2013}. Here the task is to classify electron neutrinos (signal) from muon neutrinos (background). There are 130,065 datapoints (28\% in +ve class) with 50 features to which we add a constant feature of 1's. We randomly split the data into a training (80\%) and testing (20\%) set. To compute ground truth, we collected T=400K samples using the exact reversible jump algorithm ($\epsilon = 0$). Then, we ran the approximate MH algorithm with different values of $\epsilon$ for around 3500 seconds.  We plot the risk in predictive mean of test data (estimated from 10 Markov chains) in  Fig.~\ref{fig:rjmcmc_risk}. Again we see that the lowest risk is obtained with $\epsilon > 0$.

 \begin{figure}
\centering
\includegraphics[scale=0.4]{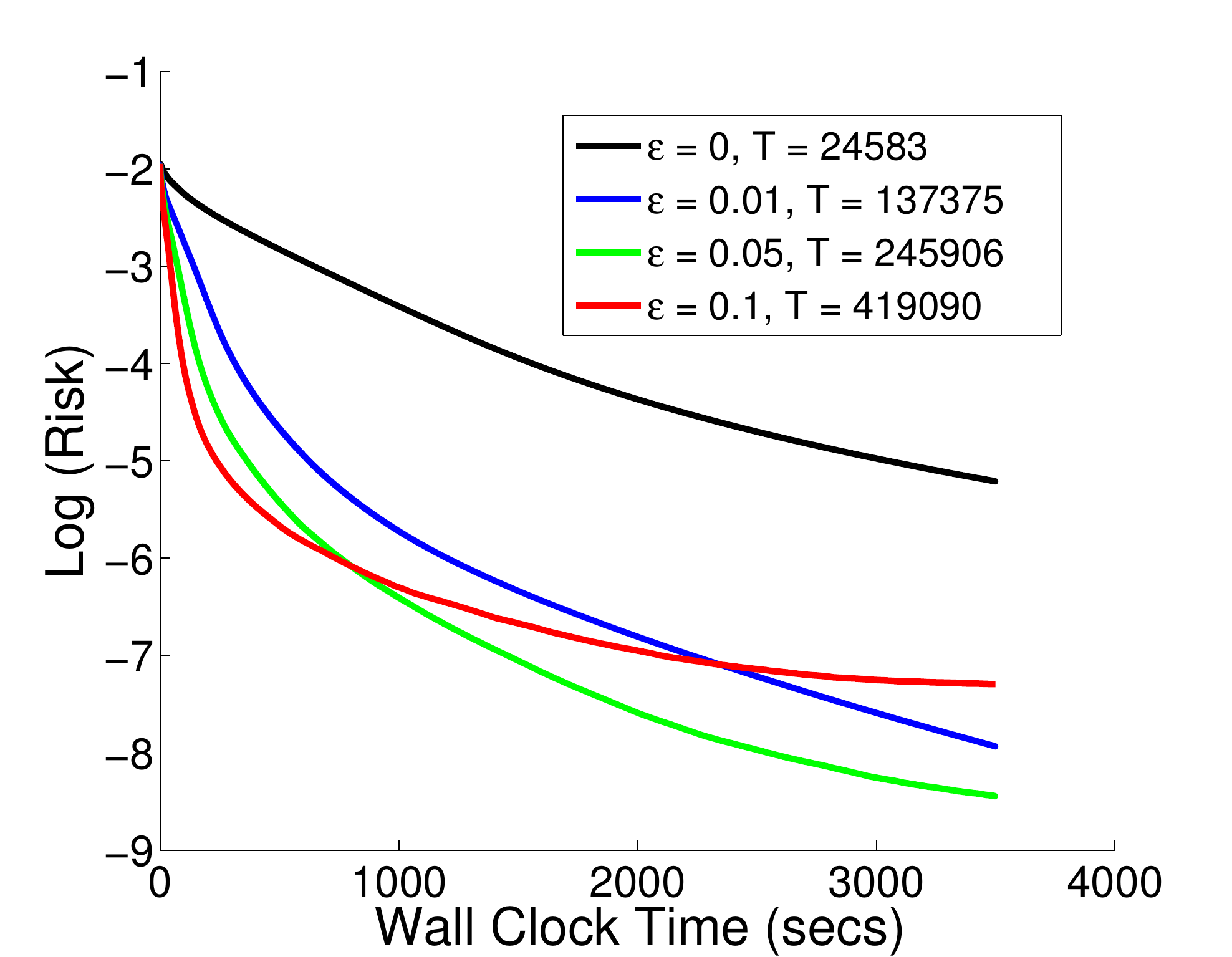}
\caption{RJMCMC: Risk in predictive mean}
\label{fig:rjmcmc_risk}
\end{figure}

The acceptance rates for the birth/death moves starts off at  $\approx$ 20\% but dies down to  $\approx$ 2\% once a good model is found. The acceptance rate for update moves is kept at $\approx$ 50\%. The model also suffers from local minima. For the plot in Fig.~\ref{fig:rjmcmc_risk}, we started with only one variable and we ended up learning models with around 12 features, giving a classification error $\approx$ 15\%. But, if we initialize the sampler with all features included and initialize $\beta$ to the MAP value, we learn models with around 45 features, but with a lower classification error $\approx$ 10\%. Both the exact reversible jump algorithm and our approximate version suffer from this problem.    We should bear this in mind when interpreting ``ground truth". However, we have observed that when initialized with the same values, we obtain similar results with the approximate algorithm and the exact algorithm (see e.g.\ Fig.~\ref{fig:rjmcmc_incl} in supplementary).

\subsection{Stochastic Gradient Langevin Dynamics}\label{sec:exp_sgld}

Finally, we apply our method to Stochastic Gradient Langevin Dynamics\citep{WellingTeh11}. In each iteration, we randomly draw a mini-batch $\mathcal{X}_n$ of size $n$, and propose:
\begin{equation}
\theta' \sim q(.|\theta,\mathcal{X}_n) = \mathcal{N} \left( \theta + \dfrac{\alpha}{2} \nabla_\theta \left\{ \dfrac{N}{n} \sum_{x \in \mathcal{X}_n} \log p (x|\theta) + \log \rho (\theta) \right\} , \alpha \right)
\end{equation}
The proposed state $\theta'$ is always accepted (without conducting any MH test). Since the acceptance probability approaches $1$ as we reduce $\alpha$, the bias from not conducting the MH test can be kept under control by using  $\alpha \approx 0$. However, we have to use a reasonably large $\alpha$ to keep the mixing rate high. This can be problematic for some distributions, because SGLD relies solely on gradients of the log density and it can be easily thrown off track by large gradients in low density regions, unless $\alpha \approx 0$.

As an example, consider an L1-regularized linear regression model. Given a dataset $\left\{x_i,y_i\right\}_{i=1}^N$ where $x_i$ are predictors and $y_i$ are targets, we use a Gaussian error model $p(y|x,\theta) \propto \exp \left\{ -\frac{\lambda}{2} (y - \theta^Tx)^2 \right\}$ and choose a Laplacian prior for the parameters $p(\theta) \propto \exp( - \lambda_0 \| \theta \|_1)$. For pedagogical reasons, we will restrict ourselves to a toy version of the problem where $\theta$ and $x$ are one dimensional. We  use a synthetic dataset with $N= 10000$ datapoints generated as $y_i = 0.5x_i + \xi$ where $\xi \sim \mathcal{N}(0,1/3)$. We choose $\lambda = 3$ and $\lambda_0 = 4950$, so that the prior is not washed out by the likelihood. The posterior density and the gradient of the log posterior are shown in figures~\ref{fig:sgld_density} and~\ref{fig:sgld_gradient} respectively.

\begin{figure}[ht]
%\begin{figure*}[htbp!]
  \centering
  \subfigure[Posterior density]{
  		\includegraphics[width=.35\textwidth]{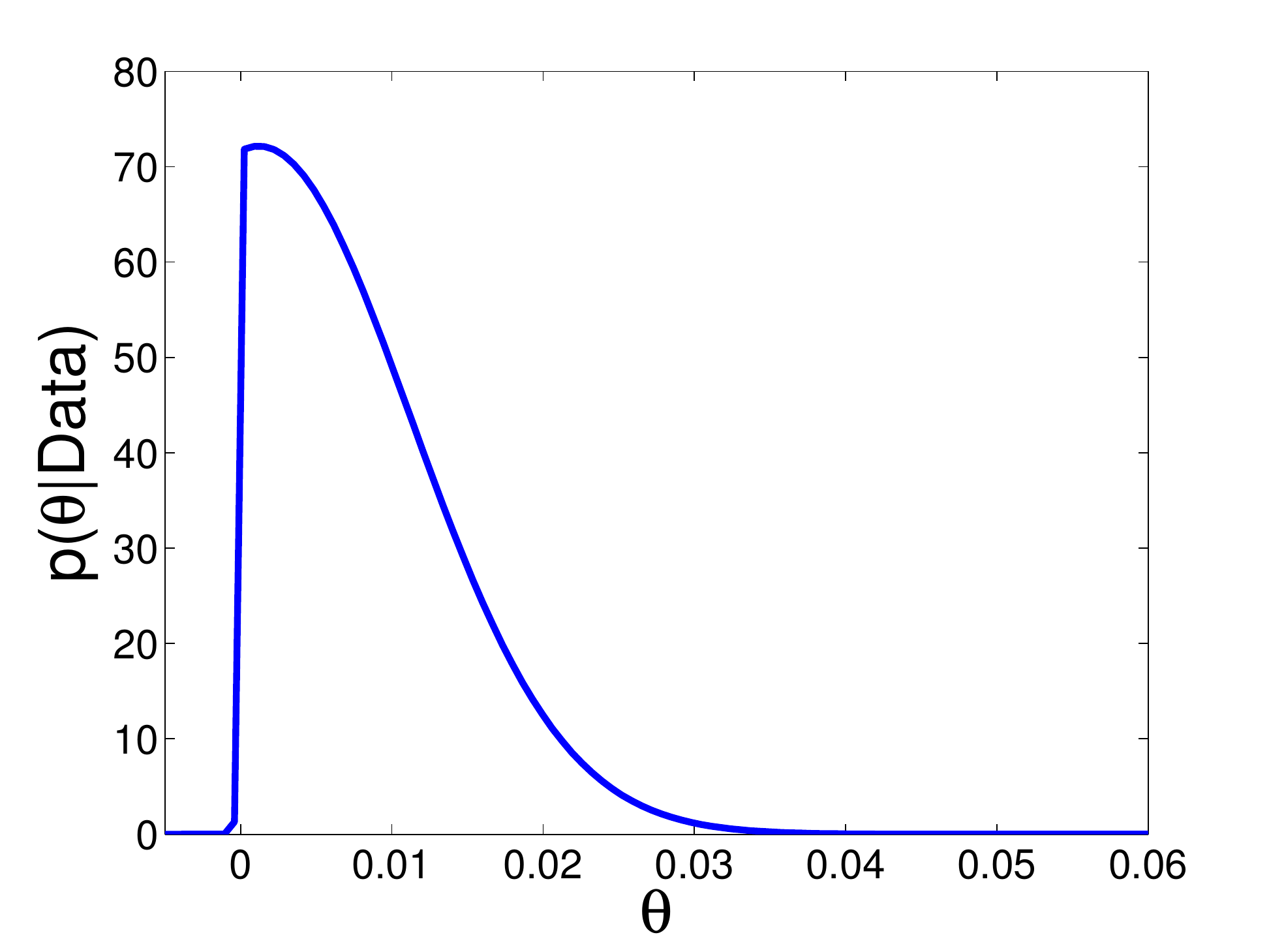}
    \label{fig:sgld_density}
  }~
  \subfigure[Gradient of log posterior]{
  		\includegraphics[width=.35\textwidth]{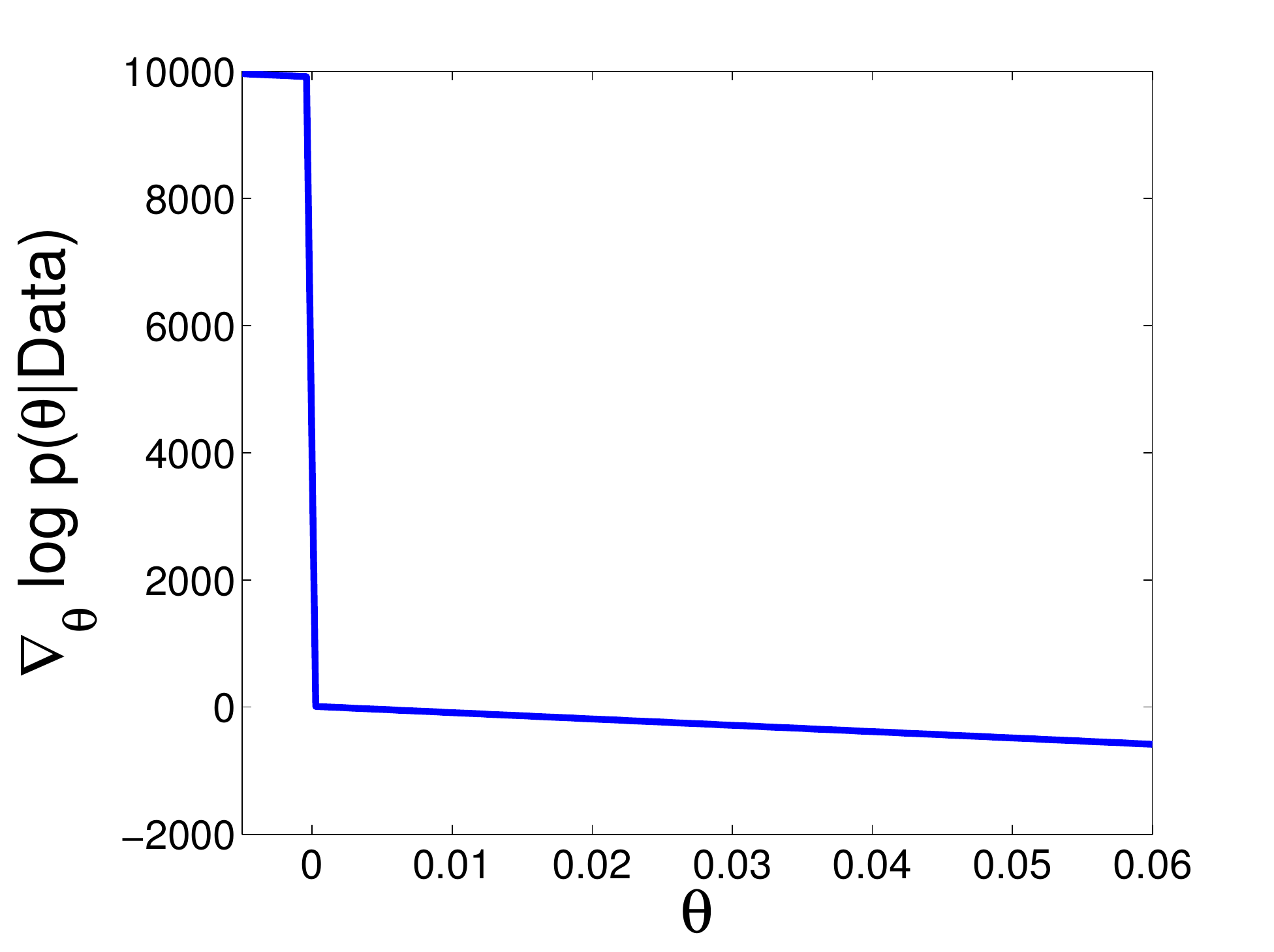}
    \label{fig:sgld_gradient}
  }\\
  \subfigure[SGLD]{
  		\includegraphics[width=.35\textwidth]{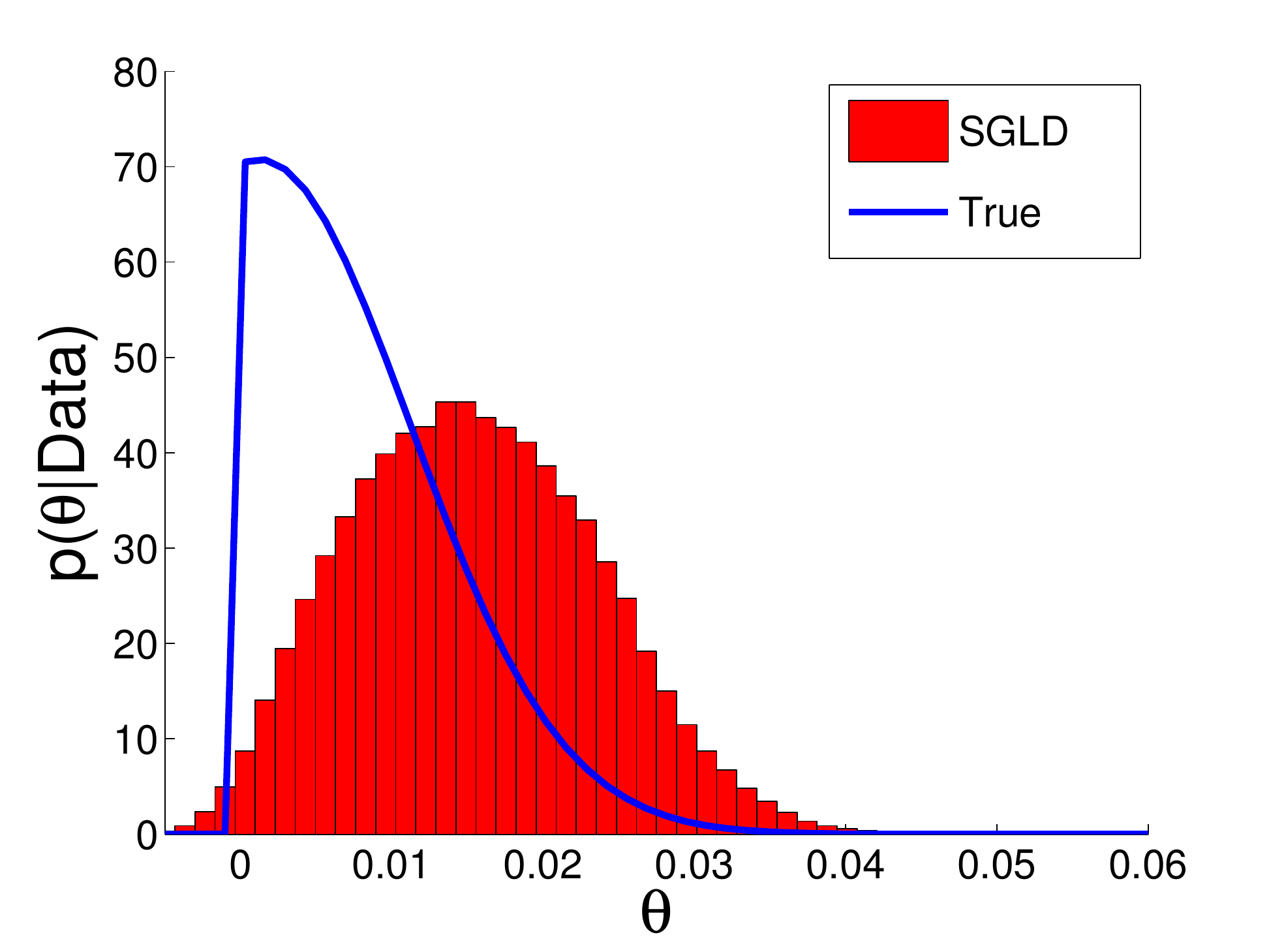}
  	  \label{fig:sgld_empsgld}
  	}~
%  \subfigure[SGLD + MH,  $\epsilon = 0$]{
%  		\includegraphics[width=.22\textwidth]{SGLD_emp_eps0}
%  	  \label{fig:sgld_emp_eps0}
%  	}\\
%  \subfigure[SGLD + MH, $\epsilon = 0.1$.]{
%  		\includegraphics[width=.22\textwidth]{SGLD_emp_eps0p1}
%  	  \label{fig:sgld_emp_eps0p1}
%  	}~
  \subfigure[SGLD + MH, $\epsilon = 0.5$. ]{
    \includegraphics[width=.35\textwidth]{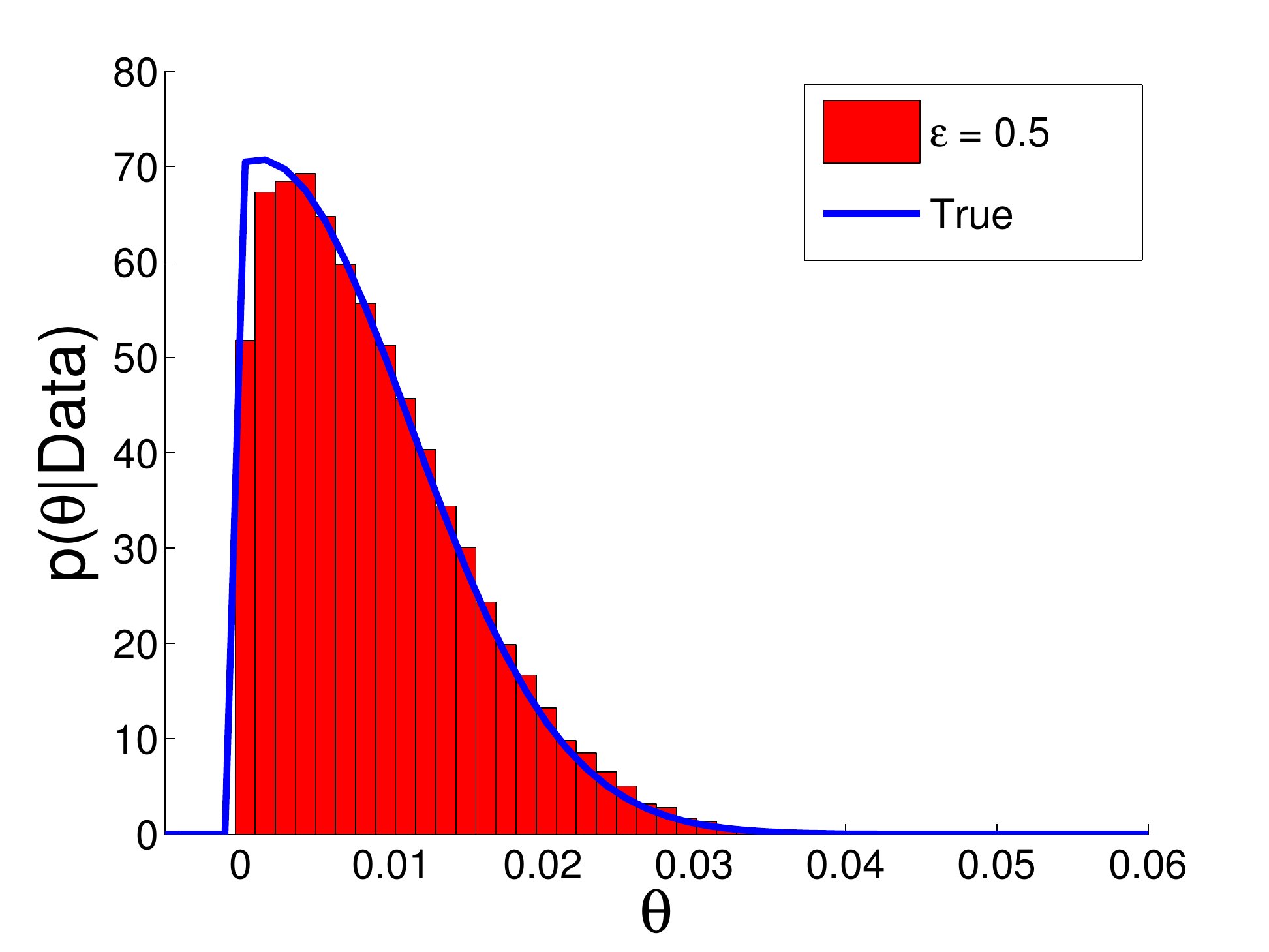}
    \label{fig:sgld_emp_eps0p5}
  }
  \caption{Pitfalls of using uncorrected SGLD}
  \label{fig:SGLD}
\end{figure}

An empirical histogram of samples obtained by running SGLD with $\alpha = 5\times10^{-6}$ is shown in Fig.~\ref{fig:sgld_empsgld}. The effect of omitting the MH test is quite severe here. When the sampler reaches the mode of the distribution, the Langevin noise occasionally throws it into the valley to the left, where the gradient is very high. This propels the sampler far off to the right, after which it takes a long time to find its way back to the mode. However, if we had used an MH accept-reject test, most of these troublesome jumps into the valley would be rejected because the density in the valley is much lower than that at the mode.

To apply an MH test, note that the SGLD proposal $q(\theta'|\theta)$  can be  considered a mixture of component kernels  $q(\theta'|\theta,\mathcal{X}_n)$ corresponding to different mini-batches. The mixture kernel will satisfy detailed balance with respect to the posterior distribution if the MH test enforces detailed balance between the posterior and each of the component kernels $q(\theta'|\theta,\mathcal{X}_n)$. Thus, we can use an MH test with
$\mu_0 = \dfrac{1}{N} \log \left[ u  \dfrac{ \rho(\theta_t) q(\theta'|\theta_t,\mathcal{X}_n)  } { \rho(\theta') q(\theta_t|\theta',\mathcal{X}_n)} \right]$.

The result of running SGLD (keeping $\alpha =5\times10^{-6}$ as before)  corrected using our approximate MH test, with $\epsilon = 0.5$, is shown in Fig.~\ref{fig:sgld_emp_eps0p5}. As expected, the MH test rejects most troublesome jumps into the valley because the density in the valley is much lower than that at the mode. The  stationary distribution is almost indistinguishable from the true posterior. Note that when $\epsilon = 0.5$, a decision is always made in the first step (using just $m = 500$ datapoints) without querying additional data sequentially.

\begin{figure}[ht]
  \centering
%  \subfigure[Error:Logistic Regression\label{fig:design_mnist_error}]{
%    \includegraphics[width=.23\textwidth]{design_mnist_error}
%    \label{fig:design_mnist_error}
%  }~
%  \subfigure[Data: Logistic Regression\label{fig:design_mnist_ratio}]{
%    \includegraphics[width=.23\textwidth]{design_mnist_ratio}
%    \label{fig:design_mnist_ratio}
%  }\\
  \subfigure[Test Average Error\label{fig:design_ica_error}]{
    \includegraphics[width=.4\textwidth]{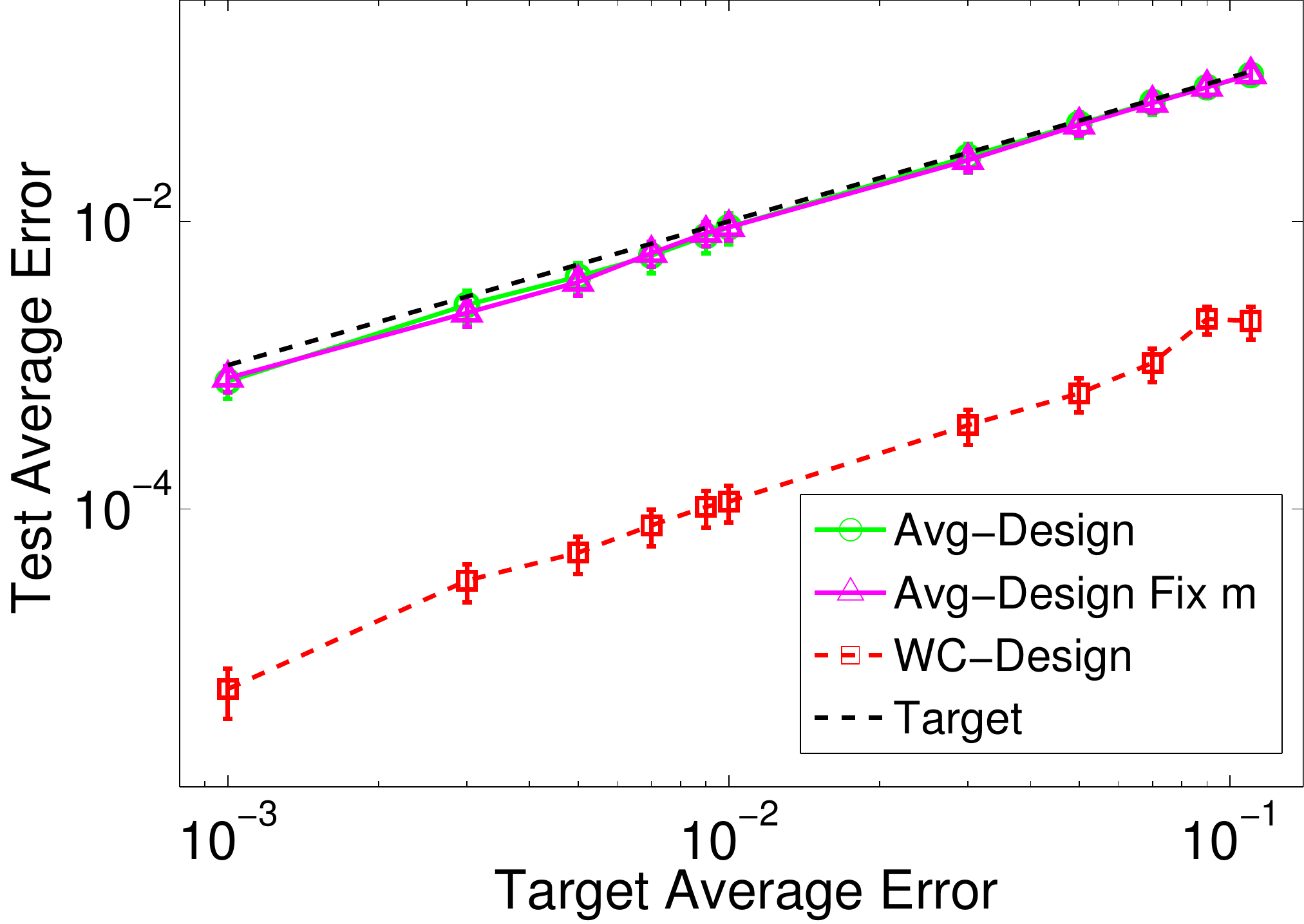}
    \label{fig:design_ica_error}
  }
  \subfigure[Average Data Usage\label{fig:design_ica_ratio}]{
    \includegraphics[width=.4\textwidth]{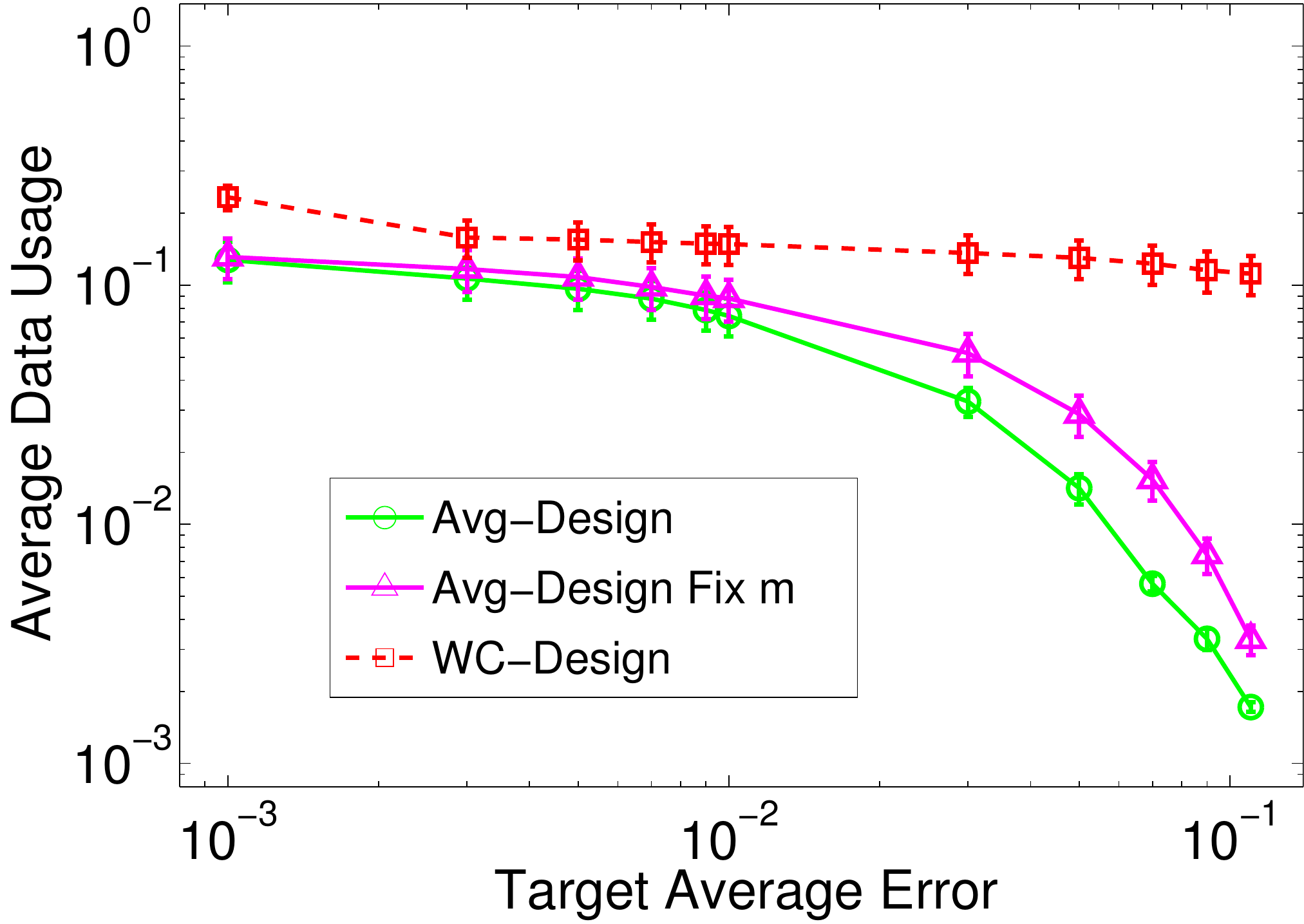}
    \label{fig:design_ica_ratio}
  }
  \caption{Test average error in $P_a$ and data usage $\eE_u[\bar{\pi}]$ for the ICA experiment using average design over both $m$ and $\epsilon$ ($\bigcirc$), with fixed $m=600$ ($\bigtriangleup$), and worst-case design ($\square$).}
  \label{fig:design}
\end{figure}

\subsection{Optimal Design of Sequential Tests}\label{sec:optdesign_exp}

We illustrate the advantages of the optimal test design proposed in Section~\ref{sec:testdesign} by applying it to the ICA experiment described in Section~\ref{sec:exp_ica}. We consider two design methods: the `average design' (Eqn.~\ref{eqn:avg_design})  and the `worst-case design'  (Eqn.~\ref{eqn:worstcase_design}). For the average design, we collected 100 samples of the Markov chain to approximate the expectation of the error over $(\theta,\theta')$. We will call these samples the training set. The worst case design does not need the training set as it does not involve the distribution of $(\theta,\theta')$. We compute the optimal $m$ and $\epsilon$ using grid search, for different values of the target training error, for both designs. We then collect a new set of 100 samples $(\theta,\theta')$ and measure the average error and data usage on this test set (Fig.~\ref{fig:design}).

For the same target error on the training set, the worst-case design gives a conservative parameter setting that achieves a much smaller error on the test set. In contrast, the average design achieves a test error that is almost the same as the target error (Fig.~\ref{fig:design_ica_error}). Therefore, it uses much less data than the worst-case design (Fig.~\ref{fig:design_ica_ratio}). 

We also analyze the performance in the case where we fix $m=600$ and only change $\epsilon$. This is a simple heuristic we recommended at the beginning of Section~\ref{sec:testdesign}. Although this usually works well, using the optimal test design ensures the best possible performance.  In this experiment, we see that when the error is large, the optimal design uses only half the data (Fig.~\ref{fig:design_ica_ratio}) used by the heuristic and is therefore twice as fast. 

\section{Conclusions and Future Work}\label{sec:conclusion}

We have taken a first step towards cutting the computational budget of the Metropolis-Hastings MCMC algorithm, which takes $O(N)$ likelihood evaluations to make the binary decision of accepting or rejecting a proposed sample. In our approach, we compute the probability that a new sample will be accepted based on a subset of the data. We increase the cardinality of the subset until a prescribed confidence level is reached. In the process we create a bias, which is more than compensated for by a reduction in variance due to the fact that we can draw more samples per unit time. Current MCMC procedures do not take these trade-offs into account. In this work we use a fixed decision threshold for accepting or rejecting a sample, but in theory a better algorithm can be obtained by adapting this  threshold over time. An adaptive algorithm can tune bias and variance contributions in such a way that at every moment our risk (the sum of squared bias and variance) is as low as possible. We leave these extensions for future work.

\section*{Acknowledgments} 

We thank Alex Ihler, Daniel Gillen, Sungjin Ahn and Babak Shahbaba for their valuable suggestions. This material is based upon work supported by the National Science Foundation under Grant No. 1216045.

\appendix

\section{Distribution of the test statistic} \label{sec:gaussian_process}

In the sequential test, we first compute the test statistic from a mini-batch of size $m$. If a decision cannot be made with this statistic, we keep increasing the mini-batch size by $m$ datapoints until we reach a decision. This procedure is guaranteed to terminate as explained in Section~\ref{sec:approxMH}.

The parameter $\epsilon$ controls the probability of making an error in a single test and not the complete sequential test. As the statistics across multiple tests are correlated with each other, we should first obtain the joint distribution of these statistics in order to estimate the error of the complete sequential test. Let $\bar{l}_j$ and $s_{l,j}$ be the sample mean and standard deviation respectively, computed using the first $j$ mini-batches. Notice that when the size of a mini-batch is large enough, e.g. $n>100$, the central limit theorem applies, and also $s_{l,j}$ is an accurate estimate of the population standard deviation. Additionally, since the degrees of freedom is high, the t-statistic in Eqn.~\ref{eqn:t-statistics}  reduces to a $z$-statistic. Therefore, it is reasonable to make the following assumptions: 

\begin{assumption}\label{ass:normal}
The joint distribution of the sequence $(\bar{l}_1,\bar{l}_2,\dots)$ follows a multivariate normal distribution.
\end{assumption}
\begin{assumption}\label{ass:exact_sg}
$s_{l} = \sg_l$, where $\sg_l=\mathrm{std}(\{l_i\})$ 
\end{assumption}

Fig.~\ref{fig:tstat_normality} shows that when $\mu=\mu_0$ the empirical marginal distribution of $t_j$ (or $z_j$) is well fitted by both a standard student-t and a standard normal distribution.
%\begin{figure*}
%  \centering
%  \subfigure[n=500]{\includegraphics[scale=0.32]{Tstat_n500}}\quad
%  \subfigure[n=5000]{\includegraphics[scale=0.32]{Tstat_n5000}}\quad
%  \subfigure[n=10000]{\includegraphics[scale=0.32]{Tstat_n10000}}
%  \caption{Theoretical (red line) and empirical distribution (blue bars) of the t-statistic under resampling $n$ datapoints without replacement from a dataset composed of digits $7$ and $9$ from the MNIST dataset (total $N = 12214$ points).\label{fig:tstat_normality}}
%\end{figure*}

Under these assumptions, we state and prove the following proposition about the joint distribution of the $z$-statistic  $\bz = (z_1, z_2, \dots)$, where $z_j \defeq (\bar{l}_j - \mu_0)/\sg_l \approx t_j$, from different tests. 

\begin{proposition}\label{prop:gaussian_process}
Given Assumption~\ref{ass:normal} and~\ref{ass:exact_sg}, the sequence $\bz$ follows a \emph{Gaussian random walk process}:
\begin{equation}
P(z_j|z_1,\dots,z_{j-1})=\cN(m_j(z_{j-1}), \sg_{z,j}^2)
\end{equation}
where
\begin{align}
m_j(z_{j-1}) &= \mu_{\mathrm{std}} \dfrac{\pi_j - \pi_{j-1}}{1-\pi_{j-1}} \dfrac{1}{\sqrt{\pi_j(1-\pi_j)}} \nn\\
& + z_{j-1}\sqrt{\dfrac{\pi_{j-1}}{\pi_j}\dfrac{1-\pi_j}{1-\pi_{j-1}}} \label{eqn:cond_mean} \\
\sg_{z,j}^2 &= \dfrac{\pi_j - \pi_{j-1}}{\pi_j(1-\pi_{j-1})} \label{eqn:cond_var}
\end{align}
with $\mu_{\mathrm{std}} = \frac{(\mu - \mu_0)\sqrt{N-1}}{\sg_l}$ being the standardized mean, and $\pi_j=jm/N$ the proportion of data in the first $j$ mini-batches.
\end{proposition}

\begin{proof}[Proof of Proposition~\ref{prop:gaussian_process}]
Denote by $x_j$ the average of $m$ $l$'s in the $j$-th mini-batch. Taking into account the fact that the $l$'s are drawn without replacement, we can compute the mean and covariance of the $x_j$'s as:
\begin{align}
\eE[x_j] &= \mu \label{eqn:mean_x} \\
\mathrm{Cov}(x_i,x_j) &= \left\{ \begin{array}{ll}
\dfrac{\sg_l^2}{m}\left(1-\dfrac{m-1}{N-1}\right) & \text{, }i=j\\
-\dfrac{\sg_l^2}{N-1} & \text{, }i \neq j\\
\end{array} \right. \label{eqn:cov_x}
\end{align}
It is trivial to derive the expression for the mean. For the covariance, we first derive the covariance matrix of single data points as
\begin{align}
\mathrm{Cov}(l_k,l_{k'}) &= \eE_{k,k'}[l_k l_{k'}] - \eE_k[l_k]\eE_{k'}[l_{k'}] \nn\\
\text{if } k = k' &\nn\\
&= \overline{l_k^2} - \mu^2 \defeq \sg_l^2 \nn\\
\text{if } k \neq k' &\nn\\
&= \eE_{k\neq k'}[l_k l_{k'}] - \mu^2 \nn\\
&= \frac{1}{N(N-1)}(\sum_{k,k'}l_k l_k' - \sum_k l_k^2) - \mu^2 \nn\\
&= \frac{N}{N-1}\mu^2 - \frac{\overline{l_k^2}}{N-1} - \mu^2 \nn\\
&= -\frac{\sg_l^2}{N-1}
\end{align}

\begin{figure*}
  \centering
  \subfigure[n=500]{\includegraphics[width=0.25\textwidth]{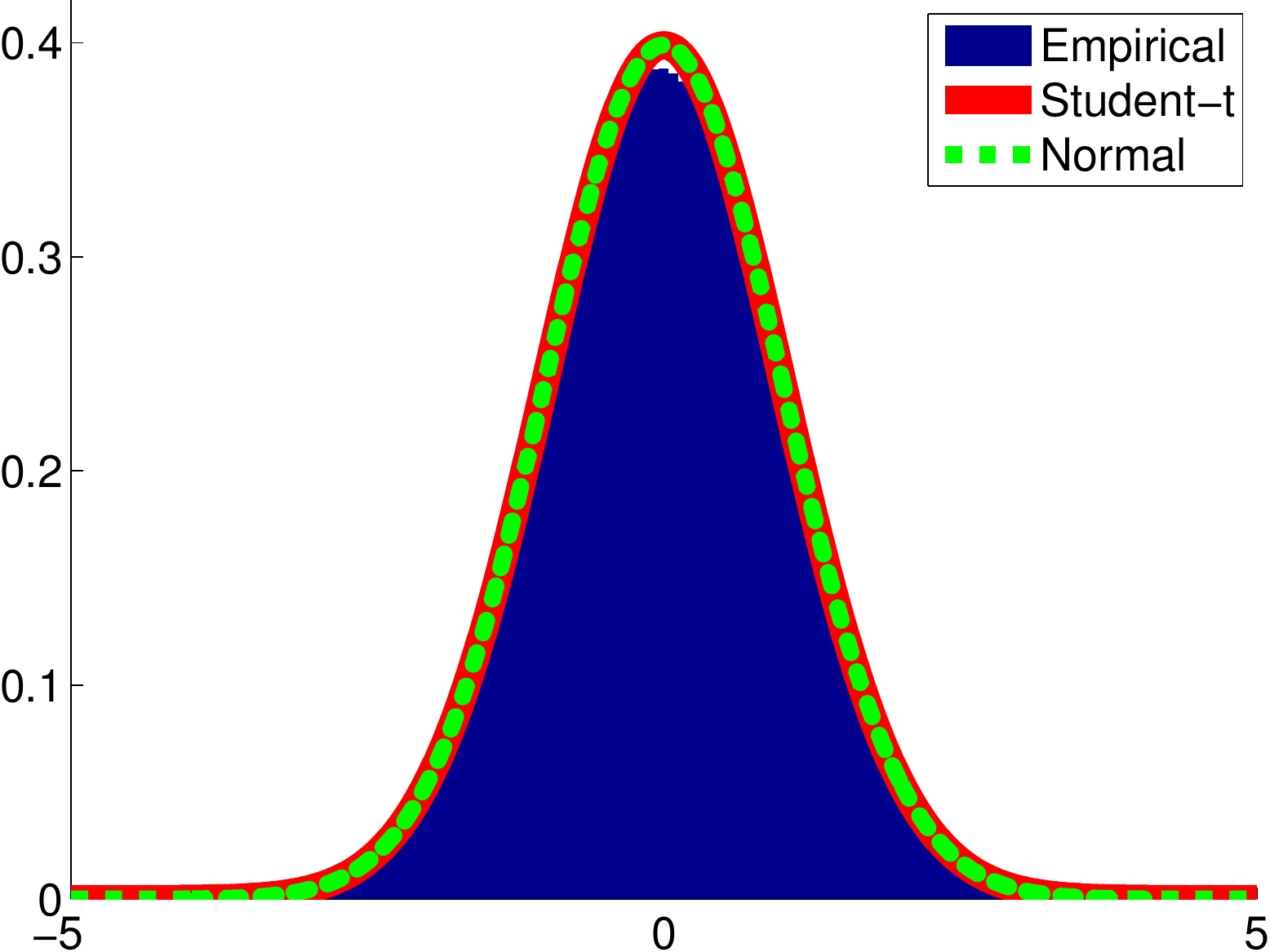}}\quad
  \subfigure[n=5000]{\includegraphics[width=0.25\textwidth]{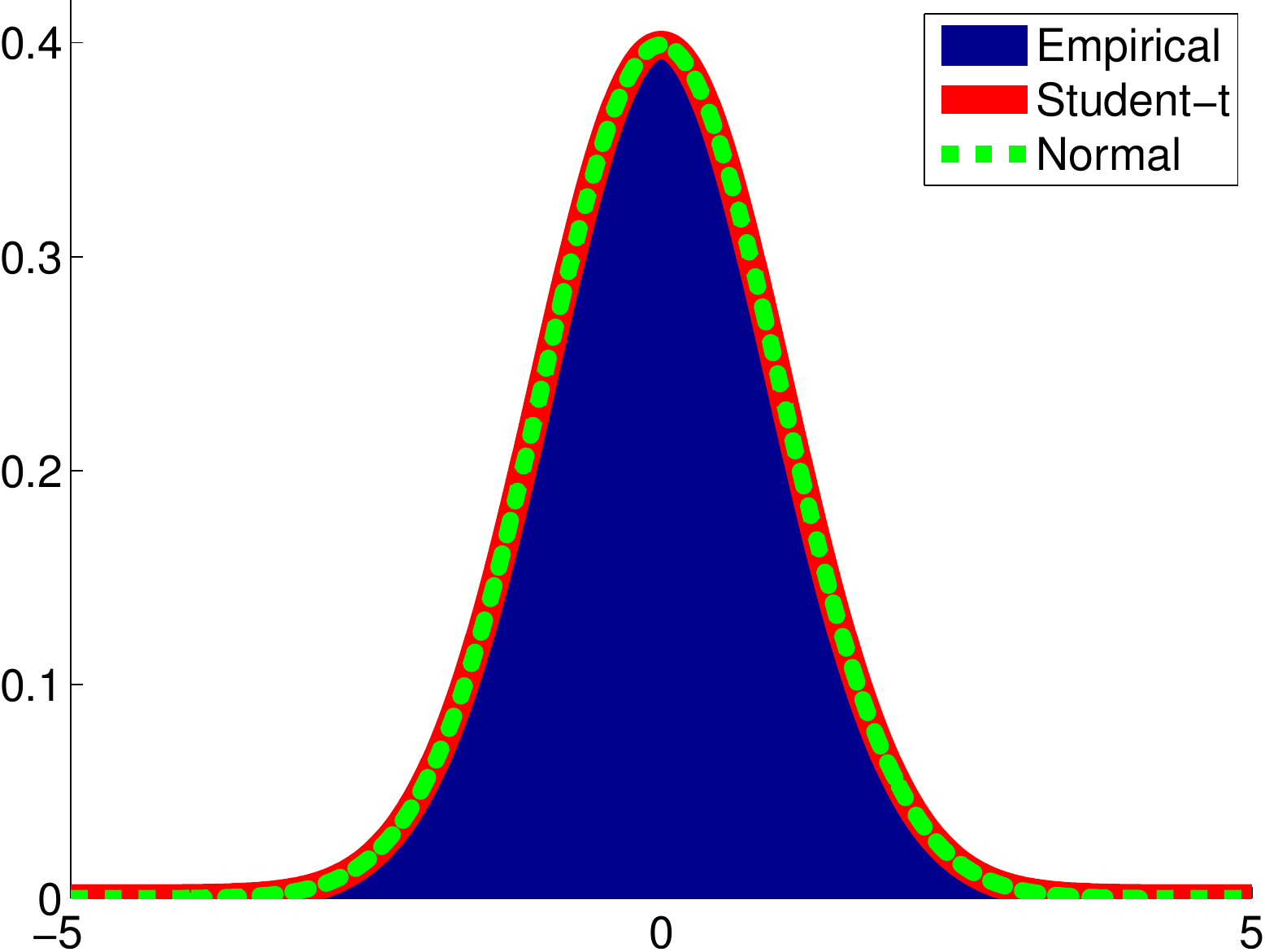}}\quad
  \subfigure[n=10000]{\includegraphics[width=0.25\textwidth]{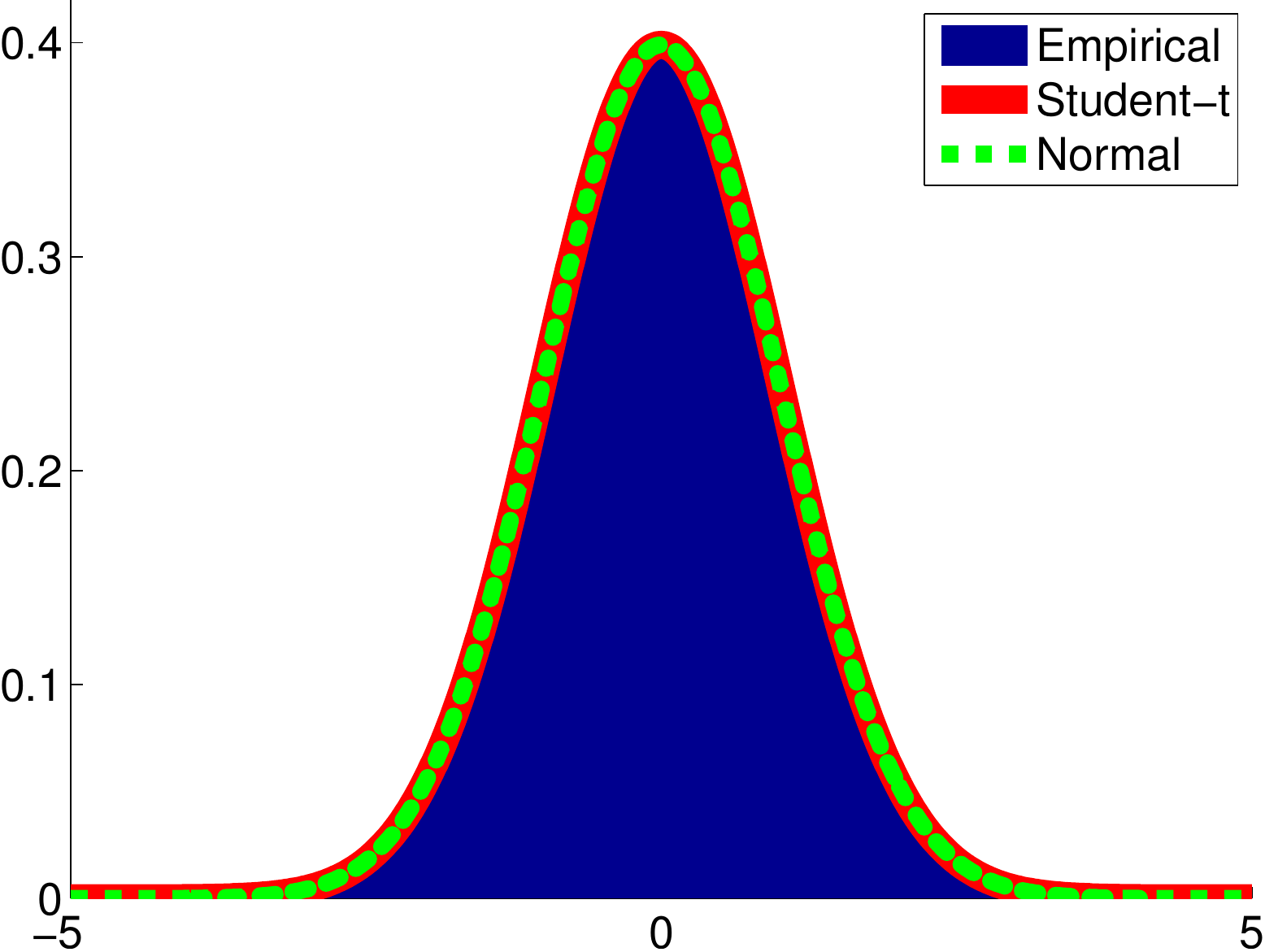}}
  \caption{Empirical distribution (blue bars) of the t-statistic under resampling $n$ datapoints without replacement from a dataset composed of digits $7$ and $9$ from the MNIST dataset (total $N = 12214$ points, mean of $l$'s is removed). Also shown are a standard normal (green dashed) and a student-t distribution with $n-1$ degrees of freedom (red solid).\label{fig:tstat_normality}}
\end{figure*}

Now, as $x_j$ can be written as a linear combination of the elements in $j$-th mini-batch as $x_j = \frac{1}{m} \boldsymbol{1}^T \boldsymbol{l}_j$, the expression for covariance in Eqn.~\ref{eqn:cov_x} follows immediately from:
\begin{equation}
\mathrm{Cov}(x_i,x_j) = \eE[x_i x_j] - \eE[x_i]\eE[x_j]
= \frac{1}{m^2}\boldsymbol{1}^T \mathrm{Cov}(\boldsymbol{l}_i \boldsymbol{l}_j^T) \boldsymbol{1}
\end{equation}

According to Assumption~\ref{ass:normal}, the joint distribution of $z_j$'s is Gaussian because $z_j$ is a linear combination of $\bar{l}_j$'s. It is however easier to derive the mean and covariance matrix of $z_j$'s by considering the vector $\bz$ as a linear function of $\bx$: $\bz = Q (\bx - \mu_0 \boldsymbol{1})$ with
\begin{equation}
Q = \left|\begin{array}{cccc}
d_1\\
& d_2 \\
& & \ddots \\
& & & d_{j}
\end{array}\right|\left|\begin{array}{cccc}
1\\
1 & 1\\
\vdots & \vdots & \ddots \\
1 & 1 & \dots & 1
\end{array}\right|
\end{equation}
where
\begin{equation}
d_j = \dfrac{\sqrt{N-1}}{j\sg_x\sqrt{\frac{N-jm}{jm}}}
\end{equation}
The mean and covariance can be computed as $\eE[\bz] = Q\boldsymbol{1}(\mu - \mu_0)$ and $\mathrm{Cov}(\bz)=Q\mathrm{Cov}(\bx)Q^T$ and the conditional distribution $P(z_j|z_1,\dots,z_{j-1})$ follows straightforwardly. We conclude the proof by plugging the definition of $\mu_{\mathrm{std}}$ and $\pi_j$ into the distribution.
\end{proof}

Fig.~\ref{fig:random_walk} shows the mean and $95\%$ confidence interval of the random walk as a function of $\pi$ with a few realizations of the $z$ sequence. Notice that as the proportion of observed data $\pi_j$ approaches 1, the mean of $z_j$ approaches infinity with a constant variance of 1. This is consistent with the fact that when we observe all the data, we will always make a correct decision.

It is also worth noting that given the standardized mean $\mu_{\mathrm{std}}$ and $\pi_j$, the process is independent of the actual size of a mini-batch $m$, population size $N$, or the variance of $l$'s $\sg_l^2$. Thus, Eqns.~\ref{eqn:cond_mean} and~\ref{eqn:cond_var} apply even if we use a different size for each mini-batch. This formulation allows us to study general properties of the sequential test, independent of any particular dataset.
%
%\begin{figure}
%  \centering
%  \subfigure[Example of the random walk of $\bz$ with $\mu_{\mathrm{std}}>0$.\label{fig:random_walk}]
%  {\includegraphics[width=.48\columnwidth]{random_walk}}
%  \subfigure[Sequential test with 3 mini-batches. Red dashed line shows the value of te bound $G$.\label{fig:random_walk_test}]
%  {\includegraphics[width=.48\columnwidth]{random_walk_test}}  
%\end{figure}

\begin{figure}
\centering
  \includegraphics[width=.6\linewidth]{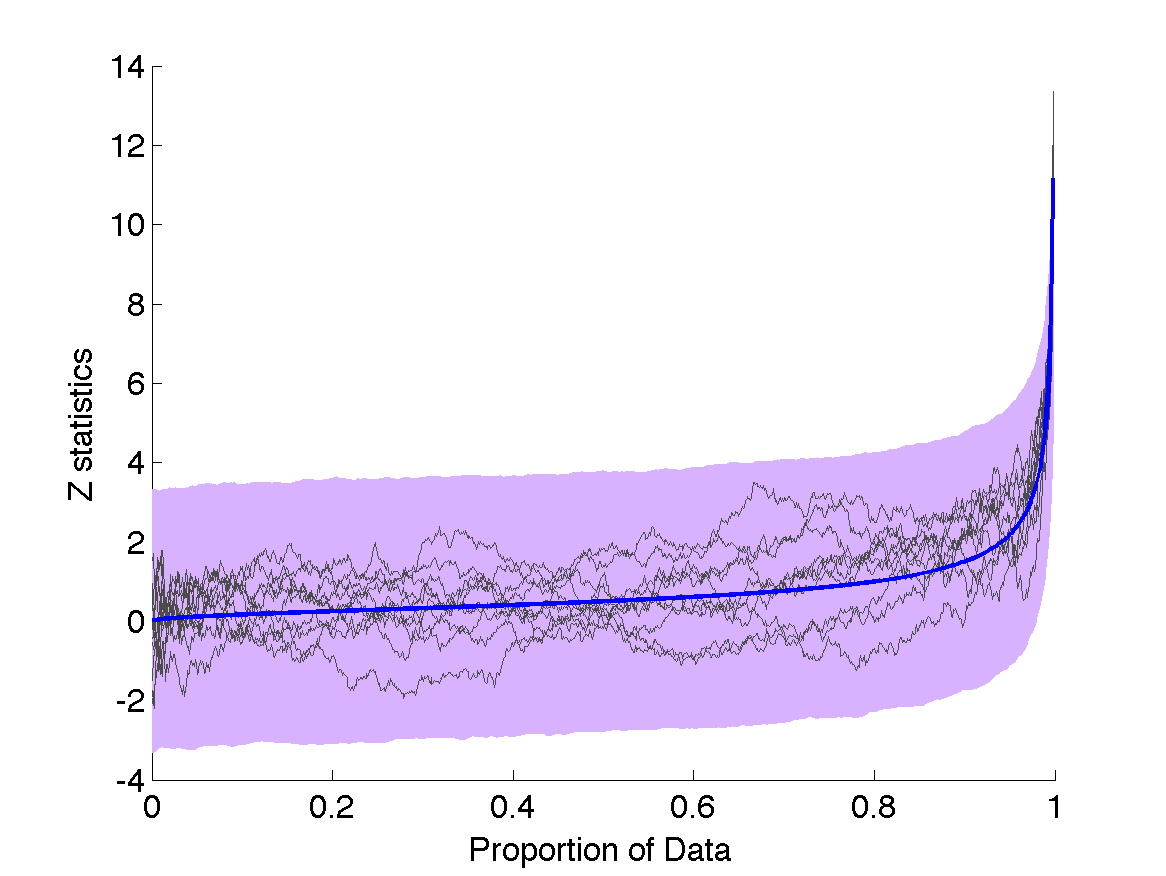}
  \caption{An example of the random walk followed by $\bz$ with $\mu_{\mathrm{std}}>0$.\label{fig:random_walk} \label{fig:random_walk}}
\end{figure}

\begin{figure}
\centering
  \includegraphics[width=.6\linewidth]{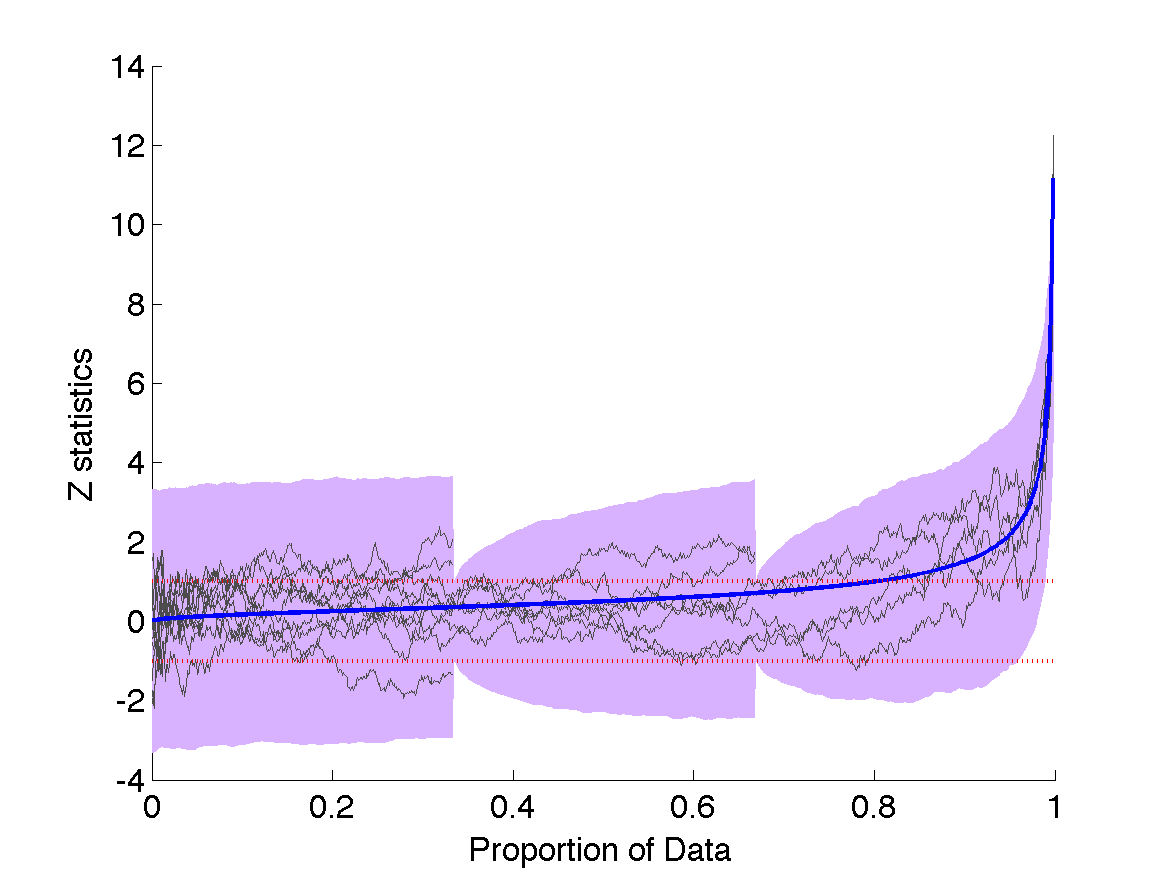}
  \caption{Sequential test with 3 mini-batches. Red dashed line is the bound $G$.\label{fig:random_walk_test} \label{fig:random_walk_test}}
\end{figure}

Applying the individual tests $\de \gtrless \epsilon \Leftrightarrow |z_j| \gtrless \Phi(1 - \epsilon) \defeq G$ at the $j$-th mini-batch corresponds to thresholding the absolute value of $z_j$ at $\pi_j$ with a bound $G$ as shown in Fig.~\ref{fig:random_walk_test}. Instead of $m$ and $\epsilon$, we will use $\pi_1 = m/N$ and $G$ as the parameters of the sequential test in the supplementary. The probability of incorrectly deciding $\mu < \mu_0$ when $\mu \geq \mu_0$ over the whole sequential test  is computed as:
\begin{equation}
\cE(\mu_{\mathrm{std}}, \pi_1, G) = \sum_{j=1}^J P(z_j < -G, |z_i|\leq G,\forall i<j) \label{eqn:type1}
\end{equation}
where $J=\lceil 1/\pi_1 \rceil$ is the maximum number of tests. Similarly the probability of incorrectly deciding $\mu \geq \mu_0$ when $\mu < \mu_0$ can be computed similarly by replacing $z_j < -G$ with $z_j > G$ in Eqn.~\ref{eqn:type1}. We can also compute the expected proportion of data that will be used in the sequential test as:
\begin{align}
&\bar{\pi}(\mu_{\mathrm{std}}, \pi_1, G) =\eE_{\bz}[\pi_{j'}] \nn\\
&= \sum_{j=1}^J \pi_jP(|z_j|>G, |z_i|\leq G,\forall i < j) \label{eqn:mean_pi}
\end{align}
where $j'$ denotes the time when the sequential test terminates.
Eqn.~\ref{eqn:type1} and~\ref{eqn:mean_pi} can be efficiently approximated together using a dynamic programming algorithm by discretizing the value of $z_j$ between $[-G,G]$. The time complexity of this algorithm is $\cO(L^2J)$ where $L$ is the number of discretized values. 

The error and data usage as functions of $\mu_{\mathrm{std}}$ are maximum in the worst case scenario when $\mu_{\mathrm{std}} \rightarrow 0 \Leftrightarrow \mu \rightarrow \mu_0$. In this case we have:
\begin{align}
& \cE(0,\pi_1,G) = \lim_{\mu_{\mathrm{std}} \rightarrow 0} \cE(\mu_{\mathrm{std}},\pi_1,G)=(1-P(j'=J))/2 \nn\\
& \defeq \cE_{\text{worst}}(\pi_1,G) \label{eqn:worst_case}
\end{align}

Figs.~\ref{fig:dynprog_vs_sim_error} and~\ref{fig:dynprog_vs_sim_ratio} show respectively that the theoretical value of the error ($\cE$) and the average data usage ($\bar{\pi}$) estimated using our dynamic programming algorithm match the simulated values. Also, note that both error and data usage drop off very fast as $\mu$ moves away from $\mu_0$.
\begin{figure}
\centering
  \includegraphics[width=.5\linewidth]{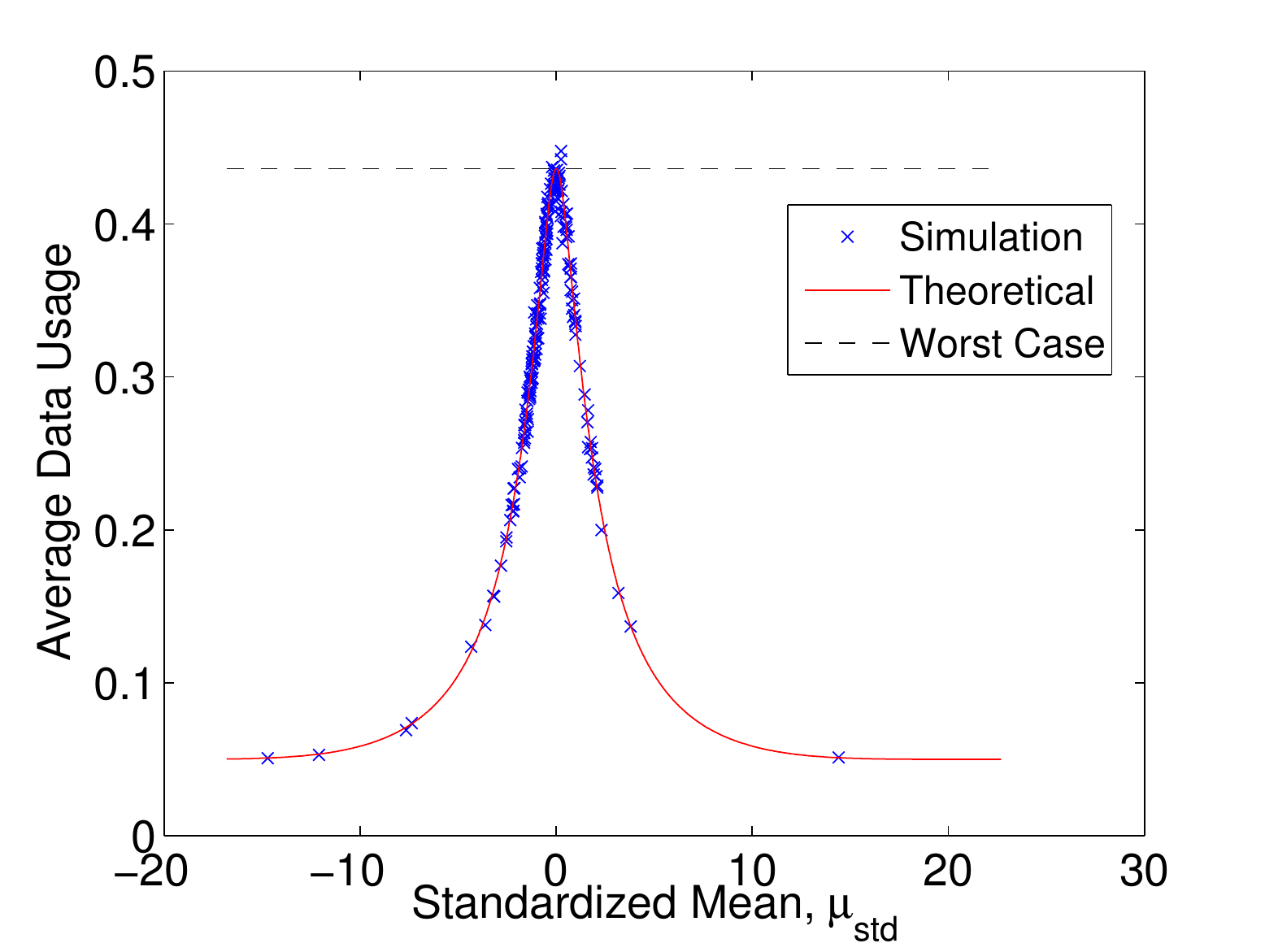}
  \caption{Average data usage $\bar{\pi}$ estimated using simulation (blue cross) and dynamic programming (red line). The worst case scenario with $\mu_{\mathrm{std}}=0$ is also shown (black dashed line). \label{fig:dynprog_vs_sim_ratio}}
\end{figure}

\section{Error in One Metropolis-Hastings Step} \label{sec:MH_step_error}
In the approximate Metropolis-Hasting test, one first draws a uniform random variable $u$, and then conducts the sequential test. As $\mu_{\mathrm{std}}$ is a function of $u$ (and $\mu$, $\sg_l$, both of which depend on $\theta$ and $\theta'$), $\cE$ measures the probability that one will make a wrong decision conditioned on $u$. One might expect that the average error in the accept/reject step of M-H using sequential test is the expected value of $\cE$ w.r.t.\ to the distribution of $u$. But in fact, we can usually achieve a significantly smaller error than a typical value of $\cE$. This is because with a varying $u$, there is some probability that $\mu > \mu_0(u)$ and also some probability that $\mu < \mu_0(u)$. Part of the error one will make given a fixed $u$ can be canceled when we marginalize out the distribution of $u$. Following the definition of $\mu_0(u)$ for M-H in Eqn.~\ref{eqn:mu_0}, we can compute the actual error in the acceptance probability as:
\begin{align}
&\Delta(\mu(\theta,\theta'), \sg_l(\theta,\theta'), \pi_1, G) = P_{a,\epsilon} - P_a \nn\\
&= \int_{0}^1 P_{\epsilon}(\mu > \mu_0(u)) \td u - \int_{0}^{P_a} \td u \nn\\
&= \int_{P_a}^1 P_{\epsilon}(\mu > \mu_0(u)) \td u - \int_{0}^{P_a} (1 - P_{\epsilon}(\mu > \mu_0(u))) \td u \nn\\
&= \int_{P_a}^1 \cE (\mu - \mu_0(u)) \td u - \int_{0}^{P_a} \cE (\mu - \mu_0(u)) \td u
\end{align}
Therefore, it is often observed in experiments (see Fig.~\ref{fig:delta_pa} for example) that when $P_a\approx 0.5$, a typical value of $\mu_{\mathrm{std}}(u)$ is close to 0, and the average value of the absolute error $|\cE|$ can be large. But due to the cancellation of errors, the actual acceptance probability $P_{a,\epsilon}$ can approximate $P_a$ very well. Fig.~\ref{fig:approximate_pa} shows the approximate $P_a$ in one step of M-H. This result also suggests that making use of some (approximate) knowledge about $\mu$ and $\sg_l$ will help us obtain a much better estimate of the error than the worst case analysis in Eqn.~\ref{eqn:worst_case}.

\begin{figure}
\centering
  \includegraphics[width=.5\linewidth]{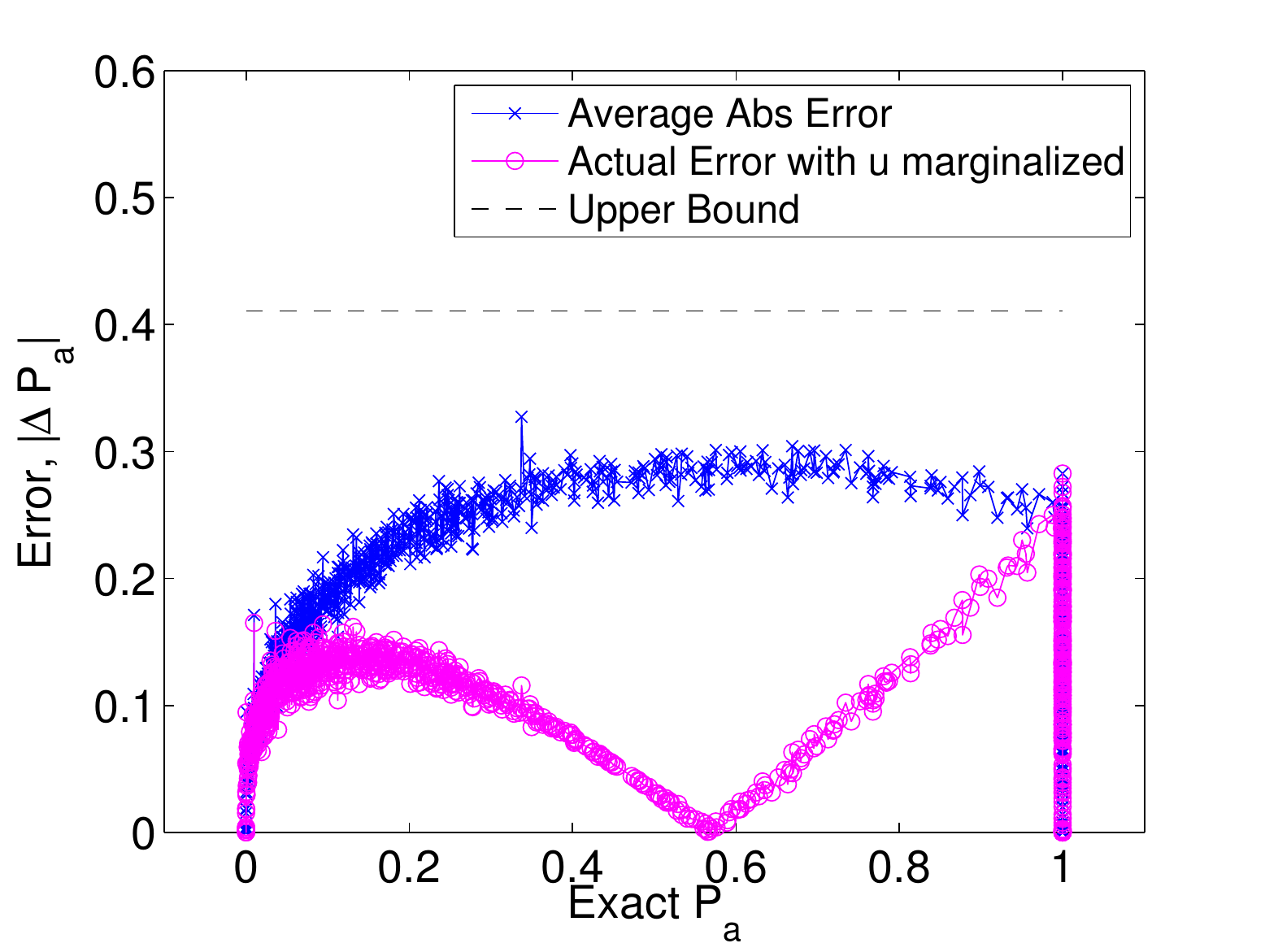}
  \caption{Error $\Delta$ in the acceptance probability (magenta circle) vs.\ exact acceptance probability $P_a$. Blue crosses are the expected value of $|\cE|$ w.r.t.\ the distribution of $u$. Black dashed line shows the upper bound.\label{fig:delta_pa}}
\end{figure}
\begin{figure}
\centering
  \includegraphics[width=.5\linewidth]{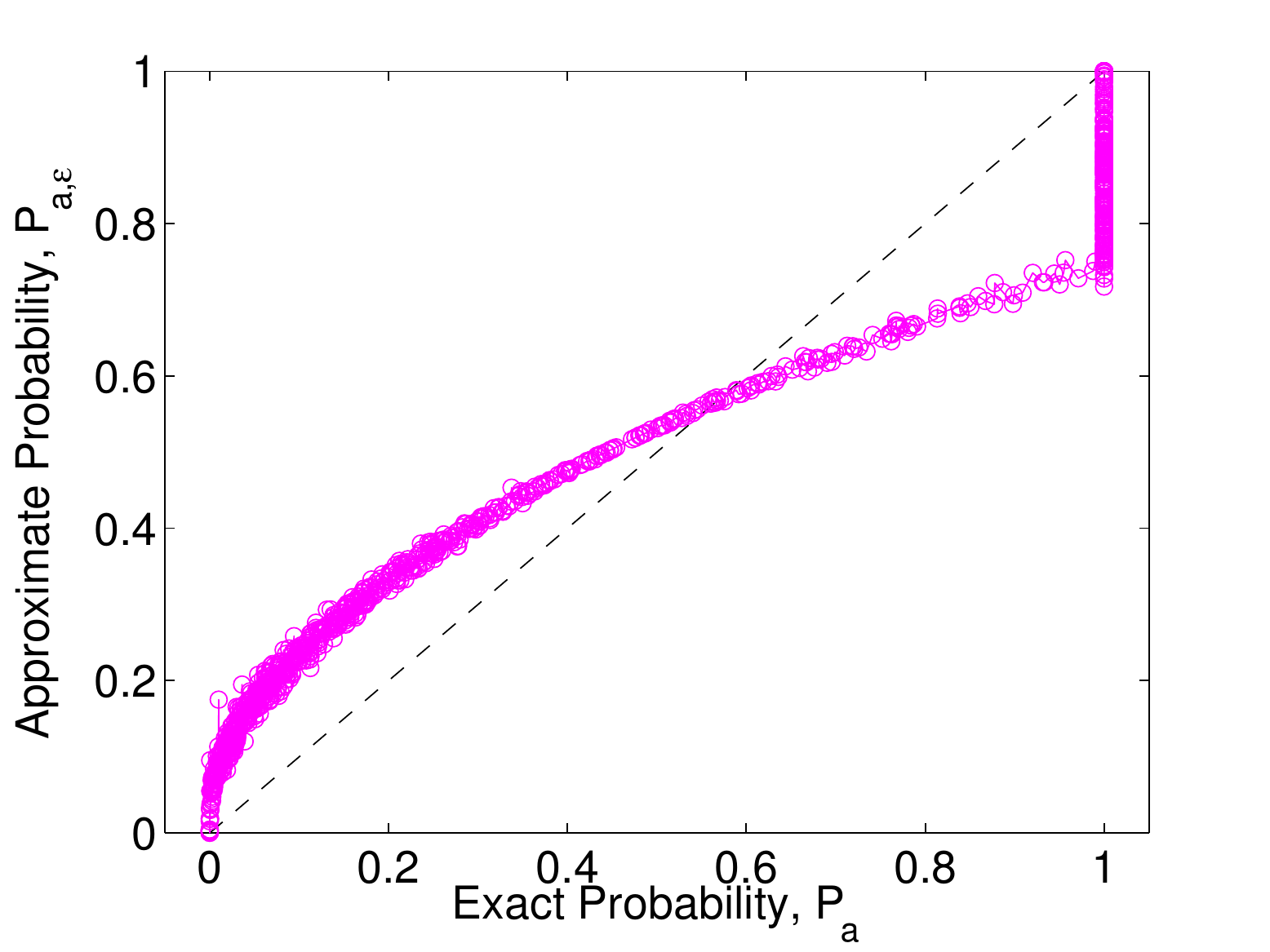}
  \caption{Approximate acceptance probability vs.\ true acceptance probability. \label{fig:approximate_pa}}
\end{figure}

\section{Proof of Theorem~\ref{thm:bound}} \label{sec:proof}
%\subsection{Notations}
%Denote $P(X)$ as a probability distribution on the state space $\mathcal{X}$. Let $\mathcal{T}$ be the transition kernel of a Markov chain, and $\tilde{\mathcal{T}}_i$ be the kernel of the approximate sampling algorithm. When we simulate a Markov chain from an initial distribution $P$, denote the distribution after $t$ steps by $P^{(t)}\defeq P\mathcal{T}^{(t)}$.

\subsection{Upper Bound Based on One Step Error}
We first prove a lemma that will be used for the proof of Theorem~\ref{thm:bound}.

\begin{lemma}\label{lem:transfer_err}
Given two transition kernels, $\mathcal{T}_0$ and $\mathcal{T}_\epsilon$, with respective stationary distributions, $\mathcal{S}_0$ and $\mathcal{S}_\epsilon$, if $\mathcal{T}_0$ satisfies the following contraction condition with a constant $\eta\in[0,1)$ for all probability distributions $P$:
\begin{equation}
d_v(P\mathcal{T}_0, \mathcal{S}_0) \leq \eta d_v(P, \mathcal{S}_0) \label{eqn:contraction}
\end{equation}
and the one step error between $\mathcal{T}_0$ and $\mathcal{T}_\epsilon$ is upper bounded uniformly with a constant $\Delta > 0$ as:
\begin{equation}
d_v(P\mathcal{T}_0, P\mathcal{T}_\epsilon) \leq \Delta, \forall P \label{eqn:one_step_error}
\end{equation}
then the distance between $\mathcal{S}_0$ and $\mathcal{S}_\epsilon$ is bounded as:
\begin{equation}
d_v(\mathcal{S}_0, \mathcal{S}_\epsilon) \leq \frac{\Delta}{1 - \eta}
\end{equation}
\end{lemma}

\begin{proof}
Consider a Markov chain with transition kernel $\mathcal{T}_\epsilon$ initialized from an arbitrary distribution $P$. Denote the distribution after $t$ steps by $P^{(t)}\defeq P\mathcal{T}_{\epsilon}^{t}$. At every time step, $t\geq 0$, we apply the transition kernel $\mathcal{T}_{\epsilon}$ on $P^{(t)}$. According to the one step error bound in Eqn.~\ref{eqn:one_step_error}, the distance between $P^{(t+1)}$ and the distribution obtained by applying $\mathcal{T}_0$ to $P^{(t)}$ is upper bounded as:
\begin{equation}
d_v(P^{(t+1)}, P^{(t)} \mathcal{T}_0) = d_v(P^{(t)} \mathcal{T}_\epsilon, P^{(t)} \mathcal{T}_0) \leq \Delta \label{eqn:one_step_error_t}
\end{equation}
Following the contraction condition of $\mathcal{T}_0$ in Eqn.~\ref{eqn:contraction}, the distance of $P^{(t)} \mathcal{T}_0$ from its stationary distribution $\mathcal{S}_0$ is less than $P^{(t)}$ as
\begin{equation}
d_v(P^{(t)} \mathcal{T}_0, \mathcal{S}_0) \leq \eta d_v(P^{(t)}, \mathcal{S}_0) \label{eqn:geo_t}
\end{equation}
Now let us use the triangle inequality to combine Eqn.~\ref{eqn:one_step_error_t} and~\ref{eqn:geo_t} to obtain an upper bounded for the distance between $P^{(t+1)}$ and $\mathcal{S}_0$:
\begin{align}
d_v(P^{(t+1)}, \mathcal{S}_0) &\leq d_v(P^{(t+1)}, P^{(t)} \mathcal{T}_0) + d_v(P^{(t)} \mathcal{T}_0, \mathcal{S}_0) \nn\\
&\leq \Delta + \eta d_v (P^{(t)}, \mathcal{S}_0) \label{eqn:triangular}
\end{align}
Let $r < 1-\eta$ be any positive constant and consider the ball $\mathcal{B}(\mathcal{S}_0,\frac{\Delta}{r})\defeq \{P: d_v(P, \mathcal{S}_0) < \frac{\Delta}{r}\}$. When $P^{(t)}$ is outside the ball, we have $\Delta \leq r d_v(P^{(t)}, S)$. Plugging this into Eqn.~\ref{eqn:triangular}, we can obtain a contraction condition for $P^{(t)}$ towards $\mathcal{S}_0$:
\begin{align}
d_v(P^{(t+1)}, \mathcal{S}_0) &\leq (r + \eta) d_v (P^{(t)}, \mathcal{S}_0)
\end{align}

So if the initial distribution $P$ is outside the ball, the Markov chain will move monotonically into the ball within a finite number of steps. Let us denote the first time it enters the ball as $t_r$. If the initial distribution is already inside the ball, we simply let $t_r=0$. We then show by induction that $P^{(t)}$ will stay inside the ball for all $t\geq t_r$.
\begin{enumerate}
\item
At $t=t_r$,  $P^{(t)}\in \mathcal{B}(\mathcal{S}_0,\frac{\Delta}{r})$ holds by the definition of $t_r$.
\item
 Assume $P^{(t)}\in \mathcal{B}(\mathcal{S}_0,\frac{\Delta}{r})$ for some $t \geq t_r$. Then, following Eqn.~\ref{eqn:triangular}, we have
\begin{align}
&d_v(P^{(t+1)}, \mathcal{S}_0) \leq \Delta + \eta \frac{\Delta}{r} = \frac{r + \eta}{r}\Delta < \frac{\Delta}{r} \nn\\
&\implies P^{(t+1 )}\in \mathcal{B}(\mathcal{S}_0,\frac{\Delta}{r})
\end{align}
\end{enumerate}
 Therefore, $P^{(t)}\in \mathcal{B}(\mathcal{S}_0,\frac{\Delta}{r})$ holds for all $t\geq t_r$. Since $P^{(t)}$ converges to  $S_\epsilon$, it follows that:
\begin{equation}
d_v(\mathcal{S}_\epsilon, \mathcal{S}_0) < \frac{\Delta}{r}, \forall r < 1 - \eta
\end{equation}
Taking the limit $r\rightarrow 1-\eta$, we prove the lemma:
\begin{equation}
d_v(\mathcal{S}_\epsilon, \mathcal{S}) \leq \frac{\Delta}{1-\eta}
\end{equation}
\end{proof}

\subsection{Proof of Theorem~\ref{thm:bound}}
We first derive an upper bound for the one step error of the approximate Metropolis-Hastings algorithm, and then use Lemma~\ref{lem:transfer_err} to prove Theorem~\ref{thm:bound}. The transition kernel of the exact Metropolis-Hastings algorithm can be written as
\begin{equation}
\mathcal{T}_0(\theta,\theta') = P_a(\theta, \theta')q(\theta'|\theta) + (1 - P_a(\theta, \theta'))\delta_D(\theta' - \theta)
\end{equation}
where $\delta_D$ is the Dirac delta function. For the approximate algorithm proposed in this paper, we use an approximate MH test with acceptance probability $\tilde{P}_{a,\epsilon}(\theta, \theta')$ where the error, $\Delta P_a \defeq P_{a,\epsilon} - P_a$, is upper bounded as $|\Delta P_a| \leq \Delta_{\text{max}}$. Now let us look at the distance between the distributions generated by one step of the exact kernel $\mathcal{T}_0$ and the approximate kernel $\mathcal{T}_\epsilon$. For any $P$,
\begin{align}
&\int_{\theta'} d\Omega(\theta') |(P\mathcal{T}_\epsilon)(\theta') - (P\mathcal{T}_0)(\theta')| \nn\\
&= \int_{\theta'} d\Omega(\theta') \left|\int_{\theta} dP(\theta) \Delta P_a(\theta,\theta')\left( q(\theta'|\theta) - \delta_D(\theta'-\theta) \right) \right| \nn\\
&\leq \Delta_{\text{max}} \int_{\theta'} d\Omega(\theta') \left|\int_\theta dP(\theta)(q(\theta'|\theta)+\delta_D(\theta'-\theta))\right| \nn\\
&= \Delta_{\text{max}} \int_{\theta'} d\Omega(\theta') \left(g_Q(\theta') + g_P(\theta')\right) = 2\Delta_{\text{max}}
\end{align}
where $g_Q(\theta')\defeq \int_\theta dP(\theta)q(\theta'|\theta)$ is the density that would be obtained by applying one step of Metropolis-Hastings without rejection. So we get an upper bound for the total variation distance as
\begin{equation}
d_v(P\mathcal{T}_\epsilon, P\mathcal{T}_0) =\ha \int_{\theta'} d\Omega(\theta')|P\mathcal{T}_\epsilon - P\mathcal{T}_0| \leq \Delta_{\text{max}}
\end{equation}
Apply Lemma~\ref{lem:transfer_err} with $\Delta = \Delta_{\text{max}}$ and we prove Theorem~\ref{thm:bound}.

\section{Optimal Sequential Test Design} \label{sec:optimal_design}

It is possible to design optimal tests that minimize the amount of  data used while keeping the error below a given tolerance. Ideally, we want to do this based on a tolerance on the error in the stationary distribution $\cS_\epsilon$. Unfortunately, this error depends on the contraction parameter, $\eta$, of the exact transition kernel, which is difficult to compute. A more practical choice is a bound $\Delta_{\text{max}}$ on the error in the acceptance probability, since the error in $\cS_\epsilon$ increases linearly with $\Delta_{\text{max}}$. 

Given $\Delta_{\text{max}}$, we want to minimize the average data usage $\bar{\pi}$ over the parameters $\epsilon$ (or $G$) and/or $m$ (or $\pi_1$) of the sequential test. Unfortunately, the error is a function of $\mu$ and $\sg_l$ which depend on $\theta$ and $\theta'$, and we cannot afford to change the test design at every iteration.

One solution is to base the design on the upper bound of the worst case error in Eqn.~\ref{eqn:worst_case} which does not rely on $\mu_{\mathrm{std}}$. But we have shown in Section~\ref{sec:MH_step_error} that this is a rather loose bound and will lead to a very conservative design that wastes the power of the sequential test. Therefore, we instead propose to design the test by bounding the expectation of the error w.r.t. the distribution $P(\mu, \sg_l)$. This leads to the following optimization problem:
\begin{align}
&\min_{\pi_1, G} \eE_{\mu, \sg_l} \eE_{u}\bar{\pi}(\mu, \sg_l, \mu_0(u), \pi_1, G) \nn\\
\text{s.t. } & \eE_{\mu, \sg_l} |\Delta(\mu, \sg_l, \pi_1, G)| \leq \Delta_{\text{max}}
\end{align}
The expectation w.r.t.\ $u$ can be computed accurately using one dimensional quadrature. For the expectation w.r.t.\ $\mu$ and $\sg_l$, we collect a set of parameter samples $(\theta,\theta')$ during burn-in, compute the corresponding $\mu$ and $\sg_l$ for each sample, and use them to empirically estimate the expectation. We can also consider collecting samples periodically and adapting the sequential design over time. Once we obtain a set of samples $\{(\mu, \sg_l)\}$, the optimization is carried out using grid search.

We have been using a constant bound $G$ across all the individual tests. This is known as the Pocock design \citep{pocock1977group}. A more flexible sequential design can be obtained by allowing $G$ to change as a function of $\pi$. \cite{wang1987approximately} proposed a bound sequence $G_j = G_0 \pi_j^{0.5-\alpha}$ where $\alpha\in[0.5, 1]$ is a free parameter. When $\alpha=0$, it reduces to the Pocock design, and when $\alpha = 1$, it reduces to O'Brien-Fleming design \citep{o1979multiple}. We can adopt this more general form in our optimization problem straightforwardly, and the grid search will now be conducted over three parameters, $\pi_1$, $G_0$, and $\alpha$.

\section{Reversible Jump MCMC}\label{sec:rjmcmc_sup}

We give a more detailed description of the different transition moves used in experiment~\ref{sec:exp_rjmcmc}. The update move is the usual MCMC move which involves changing the parameter vector $\beta$ without changing the model $\gamma$. Specifically, we randomly pick an active component $j:\gamma_j = 1$ and set $\beta_j = \beta_j + \eta$ where $\eta \sim \mathcal{N}(0,\sigma_{update})$. The birth move involves (for $k<D$) randomly picking an inactive component $j:\gamma_j = 0$ and setting $\gamma_j = 1$. We also propose a new value for $\beta_j \sim \mathcal{N}(0,\sigma_{birth})$. The birth move is paired with a corresponding death move (for $k>1$) which involves randomly picking an active component $j:\gamma_j = 1$ and setting $\gamma_j = 0$. The corresponding $\beta_j$ is discarded. The probabilities of picking these moves $p(\gamma \rightarrow \gamma')$ is the same as in \cite{chen2011bayesian}. The value of $\mu_0$ used in the MH test for different moves is given below. \\
1. Update move:
\begin{equation}
\mu_0 = \frac{1}{N} \log \left[ u \frac{\| \beta \|_1^{-k}} {\| \beta' \|_1^{-k}}\right]
\end{equation}
2. Birth move:
\begin{equation}
\mu_0 = \frac{1}{N} \log \left[ u \frac{\| \beta \|_1^{-k} p(\gamma \rightarrow \gamma') \mathcal{N}(\beta_j|0,\sigma_{birth}) (D-k)} {\| \beta' \|_1^{-(k+1)} p(\gamma' \rightarrow \gamma) \lambda k } \right]
\end{equation}
2. Death move:
%\begin{equation}
%\mu_0 = \frac{1}{N} \log \left[ u \frac{\| \beta \|_1^{-k} p(\gamma \rightarrow \gamma')  \lambda (k-1)} {\| \beta' \|_1^{-(k-1)} p(\gamma' \rightarrow \gamma) \mathcal{N}(\beta_j|0,\sigma_{birth})(D-k+1)} \right]
%\end{equation}
\begin{align}
&\mu_0 = \frac{1}{N} \times \nonumber \\ &\log \left[ u  \frac{\| \beta \|_1^{-k} p(\gamma \rightarrow \gamma')  }{\| \beta' \|_1^{-(k-1)} p(\gamma' \rightarrow \gamma) } \frac{\lambda (k-1)} { \mathcal{N}(\beta_j|0,\sigma_{birth})(D-k+1)} \right]
\end{align}

We used $\sigma_{update} = 0.01$ and  $\sigma_{birth} = 0.1$ in this experiment. As mentioned in the main text, both the exact reversible jump algorithm and our approximate version suffer from local minima. But, when initialized with the same values, we obtain similar results with both algorithms. For example, we plot the marginal posterior probability of including a feature in the model, i.e. $p(\gamma_j = 1| X_N,y_N, \lambda)$ in figure~\ref{fig:rjmcmc_incl}.

\begin{figure}
  \centering
  \subfigure{\includegraphics[scale=0.35]{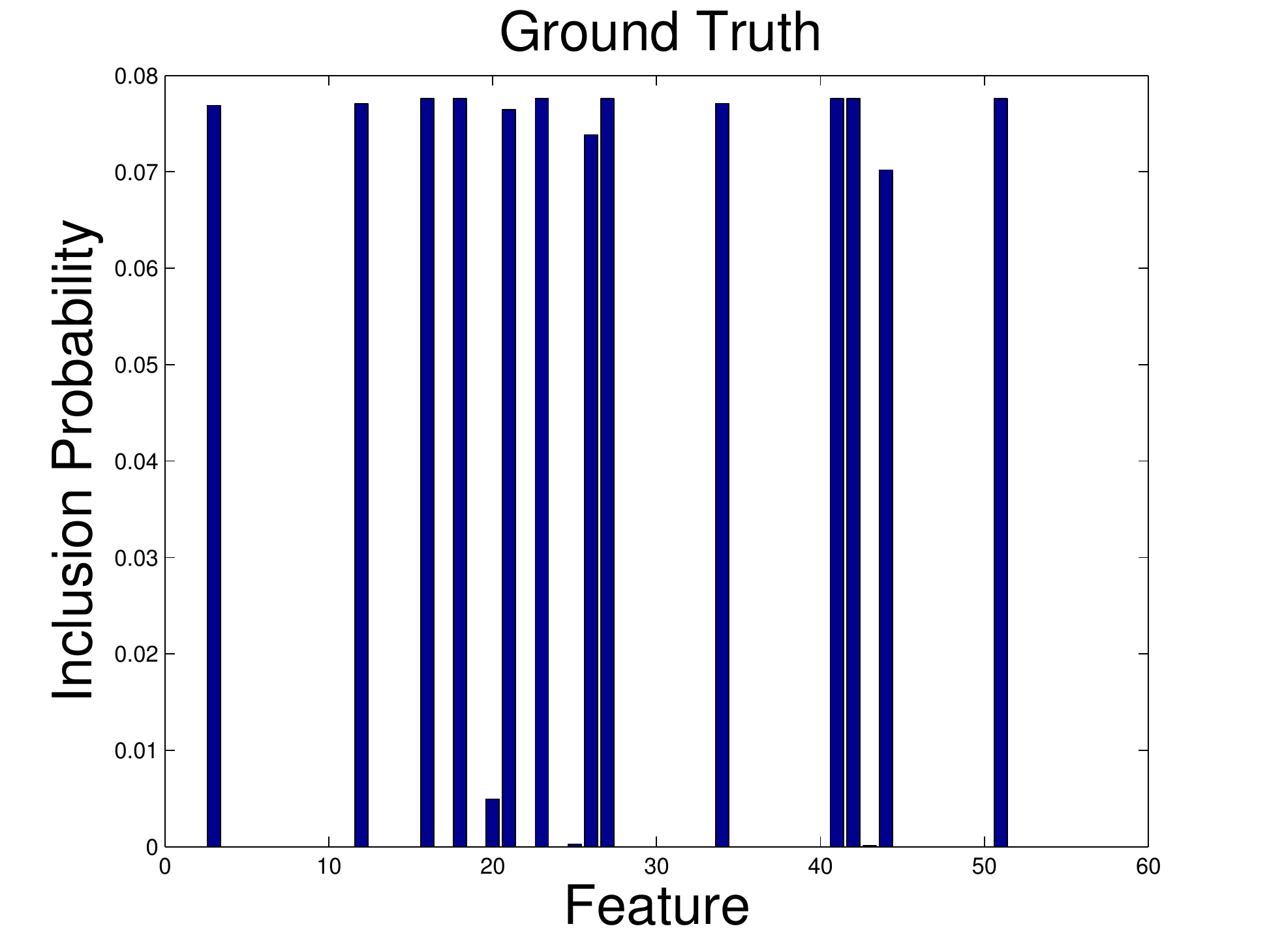}}\label{fig:rjmcmc_incl_truth}\\
{\includegraphics[scale=0.35]{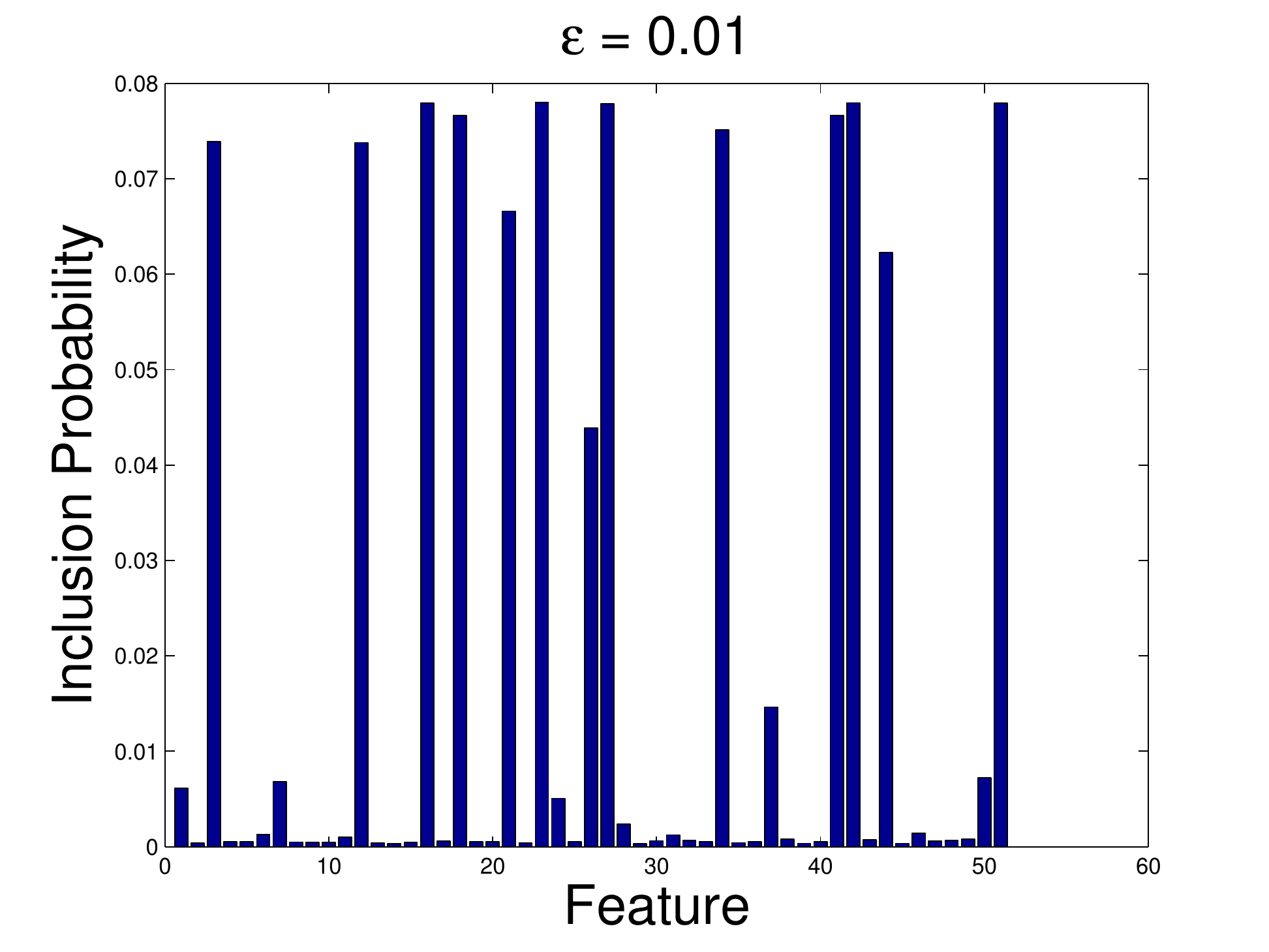}}\label{fig:rjmcmc_incl_epsp01}
  \caption{Marginal probability of features to be included in the model}
  \label{fig:rjmcmc_incl}
\end{figure}

\section{Application to Gibbs Sampling}\label{sec:gibbs}
The same sequential testing method can be applied to the Gibbs sampling algorithm for discrete models. We study a model with binary variables in this paper while the extension to multi-valued variables is also possible. Consider running a Gibbs sampler on a probability distribution over $D$ binary variables $P(X_1,\dots,X_D)$. At every iteration, it updates one variable $X_i$ using the following procedure:
\begin{enumerate}
\item
Compute the conditional probability:
\begin{equation}
P(X_i=1|x_{-i})=\frac{P(X_i=1, x_{-i})}{P(X_i=1, x_{-i})+P(X_i=0, x_{-i})}
\end{equation}
where $x_{-i}$ denotes the value of all variables other than the $i^{\text{th}}$ one.

\item
Draw $u\sim \mathrm{Uniform}[0,1]$. If $u < P(X_i=1|x_{-i})$ set $X_i=1$, otherwise set $X_i=0$.
\end{enumerate}
The condition in step 2 is equivalent to checking:
\begin{equation}
\frac{\log{u}}{\log(1-u)} < \frac{\log P(X_i=1, x_{-i})}{\log P(X_i=0, x_{-i})}
\end{equation}
When the joint distribution is expensive to compute but can be represented as a product over multiple terms, $P(X)=\prod_{n=1}^N f_n(X)$, we can apply our sequential test   to speed up the Gibbs sampling algorithm. In this case the variable $\mu_0$ and $\mu$ is given by
\begin{align}
\mu_0 &= \frac{1}{N}\frac{\log{u}}{\log(1-u)} \\
\mu &= \frac{1}{N}\sum_{n=1}^N \log \frac{f_n(X_i=1, x_{-i})}{f_n(X_i=0, x_{-i})}
\end{align}

Similar to the Metropolis-Hastings algorithm, given an upper bound in the error of the approximate conditional probability
$$
\Delta_{\text{max}} = \max_{i,x_{-i}} \left|P(X_i \text{ is assigned 1}|x_{-i}) - P(X_i=1|x_{-i})\right|
$$
we can prove the following theorem:
\begin{theorem}\label{thm:gibbs_bound}
For a Gibbs sampler with a Dobrushin coefficient $\eta\in[0,1)$ \citep[\S 7.6.2]{bremaud1999markov}, the distance between the stationary distribution and that of the approximate Gibbs sampler $S_\epsilon$ is upper bounded by
$$
d_v(S_0, S_\epsilon) \leq \frac{\Delta_{\text{max}}}{1-\eta}
$$
\end{theorem}
\begin{proof}
The proof is similar to that of Theorem~\ref{thm:bound}. We first obtain an upper bound for the one step error and then plug it into Lemma~\ref{lem:transfer_err}.

The exact transition kernel of the Gibbs sampler for variable $X_i$ can be represented by a matrix $\mathcal{T}_{0,i}$ of size $2^D\times 2^D$:
\begin{equation}
\mathcal{T}_{0,i}(x, y) = \left\{\begin{array}{rl}
0 & \mbox {if } x_{-i} \neq y_{-i} \\
P(Y_i = y_i | y_{-i}) & \mbox {otherwise}
\end{array}\right.
\end{equation}
where $1\leq i\leq N, x, y \in \{0, 1\}^D$. The approximate transition kernel $\mathcal{T}_{\epsilon,i}$ can be represented similarly as
\begin{equation}
\mathcal{T}_{\epsilon, i}(x, y) = \left\{\begin{array}{rl}
0 & \mbox {if } x_{-i} \neq y_{-i} \\
P_\epsilon(Y_i = y_i | y_{-i}) & \mbox {otherwise}
\end{array}\right.
\end{equation}
where $P_\epsilon$ is the approximate conditional distribution. Define the approximation error $\Delta \mathcal{T}_i(x,y) \defeq \mathcal{T}_{\epsilon,i}(x,y) - \mathcal{T}_{0,i}(x,y)$. We know that $\Delta \mathcal{T}_i(x,y) = 0$ if $y_{-i}\neq x_{-i}$ and it is upper bounded by $\Delta_{\text{max}}$ from the premise of Theorem~\ref{thm:gibbs_bound}.

Notice that the total variation distance reduces to a half of the $L_1$ distance for discrete distributions. For any distribution $P$, the one step error is bounded as
\begin{align}
& d_v(P\mathcal{T}_{\epsilon,i}, P\mathcal{T}_{0,i})
= \ha \|P\mathcal{T}_{\epsilon,i} - P\mathcal{T}_{0,i}\|_1 \nn\\
&=\ha \sum_y \left|\sum_x P(x) \Delta \mathcal{T}(x,y)\right| \nn\\
&=\ha \sum_y \left|\sum_{x_i\in\{0,1\}} P(x_i,y_{-i}) \Delta P(x_i|y_{-i})\right| \nn\\
&\leq \ha \Delta_{\text{max}} \sum_y \left|P(Y_{-i}=y_{-i})\right| \nn\\
&= \Delta_{\text{max}}
\end{align}

For a Gibbs sampling algorithm, we have the contraction condition \citep[\S 7.6.2]{bremaud1999markov}:
\begin{equation}
d_v(P\mathcal{T} ,S) \leq \eta d_v(P,S)
\end{equation}
Plug $\Delta = \Delta_{\text{max}}$ and $\eta$ into Lemma~\ref{lem:transfer_err} and we obtain the conclusion.
\end{proof}

\subsection{Experiments on Markov Random Fields}
We illustrate the performance of our approximate Gibbs sampling algorithm on a synthetic Markov Random Field. The model under consideration has $D=100$ binary variables and they are densely connected by potential functions of three variables $\psi_{i,j,k}(X_i,X_j,X_k), \forall i\neq j\neq k$. There are $D(D-1)(D-2)/6$ potential functions in total (we assume potential functions with permuted indices in the argument are the same potential function), and every function has $2^3=8$ values. The entries in the potential function tables are drawn randomly from a log-normal distribution, $\log \psi_{i,j,k}(X_i,X_j,X_k) \sim \mathcal{N}(0, 0.02)$. To draw a Gibbs sample for one variable $X_i$ we have to compute $(D-1)(D-2)/2=4851$ pairs of potential functions as
\begin{equation}
\frac{P(X_i=1|x_{-i})}{P(X_i=0|x_{-i})}=\frac{\prod_{i\neq j\neq k}\psi_{i,j,k}(X_i=1,x_j,x_k)}{\prod_{i\neq j\neq k}\psi_{i,j,k}(X_i=0,x_j,x_k)}
\end{equation}
The approximate methods use a mini-batches of $500$ pairs of potential functions at a time. We compare the exact Gibbs sampling algorithm with approximate versions with $\epsilon \in \{0.01, 0.05, 0.1, 0.15, 0.2, 0.25\}$.

To measure the performance in approximating $P(X)$ with samples ${x_t}$, the ideal metric would be a distance between the empirical joint distribution and $P$. Since it is impossible to store all the $2^{100}$ probabilities, we instead repeatedly draw $M=1600$ subsets of $5$ variables, $\{s_m\}_{m=1}^M, s_m\subset \{1,\dots,D\}, |s_m|=5$, and compute the average $L_1$ distance of the joint distribution on these subsets between the empirical distribution and $P$:
\begin{equation}
\mathrm{Error} = \frac{1}{M}\sum_{s_m}\|\hat{P}(X_{s_m})-P(X_{s_m})\|_1
\end{equation}
The true $P$ is estimated by running exact Gibbs chains for a long time. We show the empirical conditional probability obtained by our approximate algorithms (percentage of $X_i$ being assigned $1$) for different $\epsilon$ in Fig.~\ref{fig:gibbs_cond_prob}. It tends to underestimate large probabilities and overestimate on the other end. When $\epsilon=0.01$, the observed maximum error is within $0.01$.

\begin{figure}[tb!]
\centering
\includegraphics[width=.4\linewidth]{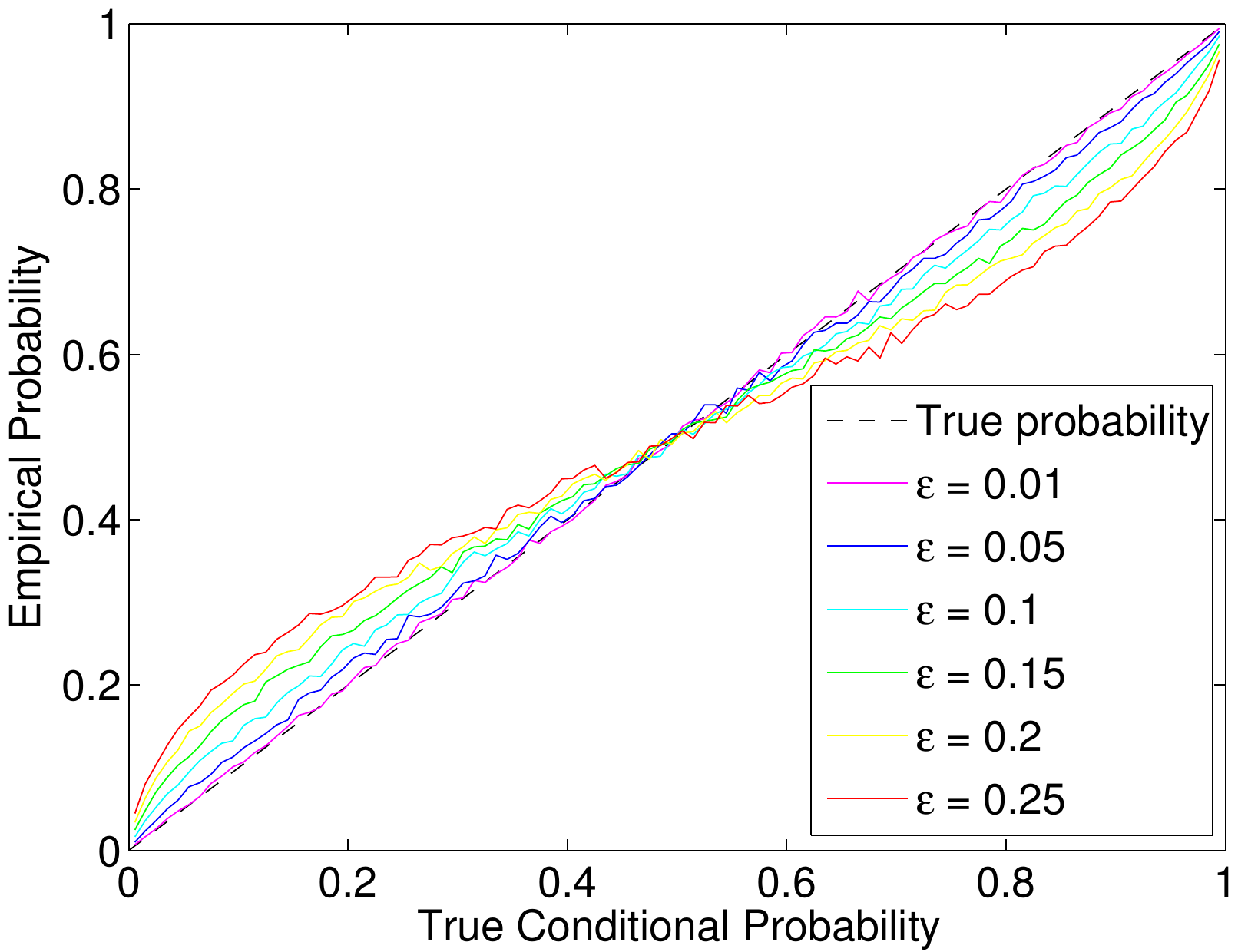}
\caption{Empirical conditional probability vs exact conditional probability for different values of $\epsilon$. The dotted black line shows the result for exact Gibbs sampling.}
\label{fig:gibbs_cond_prob}
\end{figure}
\begin{figure}[tb!]
\centering
\includegraphics[width=.4\linewidth]{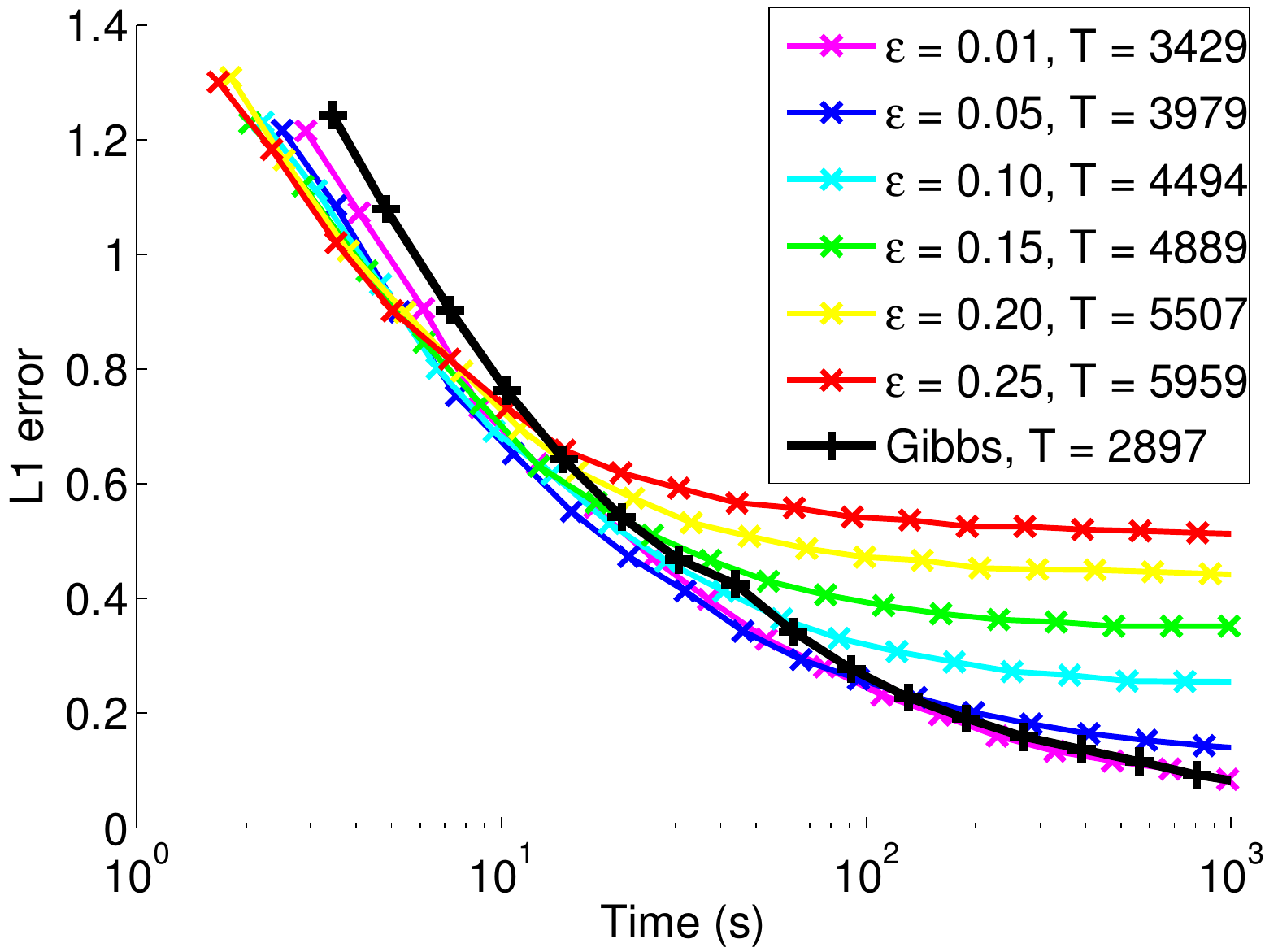}
\caption{Average $L_1$ error in the joint distribution over cliques of $5$ variables vs running time for different values of $\epsilon$. The black line shows the error of Gibbs sampler with an exact acceptance probability. $T$ in the legend indicates the number of samples obtained after $1000$ seconds.}
\label{fig:gibbs_across_t}
\end{figure}

Fig.~\ref{fig:gibbs_across_t} shows the error for different $\epsilon$ as a function of the running time. For small $\epsilon$, we use fewer mini-batches per iteration and thus generate more samples in the same amount of time than the exact Gibbs sampler. So the error decays faster in the beginning. As more samples are collected the variance is reduced. We see that these plots converge towards their bias floor while the exact Gibbs sampler out-performs all the approximate methods at around $1000$ seconds.

%We verify the above observations in Fig.~\ref{fig:gibbs_across_p0}, where the methods with various $\epsilon$ are compared with each other at three time points. At an early stage, even approximate methods with high asymptotic bias (e.g. $\epsilon=0.75$) beat the Gibbs sampler. The valley of the plot moves gradually to the right end as time increases. We find that $\epsilon=0.975$ works well over all time scales in this experiments. [{\bf this plot can perhaps go. It does not look convincing because either the error is ridiculously large or the dip is very close to 1. In general we should spend more time on the Bayesian posterior sampler I think.}]. 

%\begin{figure}[tb!]
%\label{fig:gibbs_across_p0}
%\centering
%\includegraphics[width=.95\linewidth]{gibbs_D5_across_p0}
%\caption{Comparison of average $L_1$ error at different values of $\epsilon$ at three time points: $10$, $100$, and $1000$ seconds. The data at $\epsilon=1$ is obtained using the exact Gibbs sampler.}
%\end{figure}

%\bibliography{Refs}
%\bibliographystyle{icml2013}

%\end{document}

\bibliography{Refs}
\bibliographystyle{abbrvnat}

\end{document}